\documentclass[twoside,11pt]{article}
\usepackage{amsfonts}
\usepackage{mathrsfs}
\usepackage[utf8]{inputenc} 
\usepackage{float}
\usepackage{enumitem}
\usepackage{amsmath,amssymb}

\usepackage{latexsym}
\usepackage{longtable}
\usepackage{makecell}
\usepackage{mathrsfs}
\newtheorem{theorem}{Theorem}
\newtheorem{lemma}{Lemma}
\newtheorem{assumption}{Assumption}

\usepackage{booktabs}
\usepackage{epstopdf}
\usepackage{tabularx}
\usepackage{tikz}
\usetikzlibrary{shapes.geometric, arrows}
\usetikzlibrary{positioning}
\tikzstyle{startstop} = [rectangle, rounded corners, minimum width=3cm, minimum height=1cm, text centered, draw=black, fill=red!30]
\tikzstyle{process} = [rectangle, minimum width=3cm, minimum height=1cm, text centered, draw=black, fill=blue!30]
\tikzstyle{io} = [trapezium, trapezium left angle=70, trapezium right angle=110, minimum width=3cm, minimum height=1cm, text centered, draw=black, fill=green!30]
\tikzstyle{arrow} = [thick,->,>=stealth]
\usepackage{geometry}  
\usepackage[active]{srcltx}
\usepackage{graphicx}
\usepackage{booktabs}
\usepackage{multirow}
\usepackage{siunitx}
\usepackage{xurl} 
\usepackage{amsmath}
\usepackage{setspace}
\usepackage{graphicx} 
\usepackage{algorithm}
\usepackage{algorithmic}
\usepackage{natbib}

\usepackage{subcaption} 
\usepackage[
    bookmarks=true,         
    bookmarksnumbered=true, 
    colorlinks=true, pdfstartview=FitV, linkcolor=blue, citecolor=blue,
    urlcolor=blue]{hyperref}

\usepackage{fancyhdr}
\usepackage{lipsum} 
\usepackage{listings} 
\usepackage{xcolor} 
\definecolor{codegreen}{rgb}{0,0.6,0}
\definecolor{codegray}{rgb}{0.5,0.5,0.5}
\definecolor{codepurple}{rgb}{0.58,0,0.82}
\definecolor{backcolour}{rgb}{0.95,0.95,0.92}

\lstdefinestyle{mystyle}{
    backgroundcolor=\color{backcolour},
    commentstyle=\color{codegreen},
    keywordstyle=\color{magenta},
    numberstyle=\tiny\color{codegray},
    stringstyle=\color{codepurple},
    basicstyle=\footnotesize\ttfamily,
    breakatwhitespace=false,
    breaklines=true,
    captionpos=b,
    keepspaces=true,
    numbers=left,
    numbersep=5pt,
    showspaces=false,
    showstringspaces=false,
    showtabs=false,
    tabsize=2
}

\providecommand{\U}[1]{\protect\rule{.1in}{.1in}}
\newtheorem {proposition}{Proposition}[section]
\newtheorem {corollary}{Corollary}[section]

\newtheorem{definition}{Definition}[section]

\newtheorem{remark}{Remark}[section]

\newcommand{\E}{\mathbb{E}}
\newenvironment{proof}[1][Proof]{\textbf{#1.} }{\
\rule{0.5em}{0.5em}}

\DeclareMathOperator*{\argmax}{arg\,max}
\DeclareMathOperator{\Var}{Var}

\usepackage{listings}
\usepackage{titlesec}
\titleformat{\section}
{\normalfont\Large\bfseries}{\thesection.}{1em}{}
\usepackage{subcaption}

\renewenvironment{enumerate}{
  \begin{list}{}{
    \setlength{\labelwidth}{0pt}
    \setlength{\leftmargin}{2em}
    \setlength{\itemindent}{0pt}
    \setlength{\itemsep}{0.5em}
    \setlength{\parsep}{0pt}
    \setlength{\topsep}{0pt}
    \setlength{\partopsep}{0pt}
    \setlength{\listparindent}{2em}
    \setlength{\labelsep}{0pt}
    \setlength{\rightmargin}{0pt}
  }
}{
  \end{list}
}

\begin{document}
\date{}
\title{\Large \textbf{Copula Discrepancy: Benchmarking Dependence Structure}}
\vspace{1ex}
\author{Agnideep Aich${ }^{1}$\thanks{Corresponding author: Agnideep Aich, \texttt{agnideep.aich1@louisiana.edu}, ORCID: \href{https://orcid.org/0000-0003-4432-1140}{0000-0003-4432-1140}}
 \hspace{0pt} and Ashit Baran Aich${ }^{2}$ 
\\ ${ }^{1}$ Department of Mathematics, University of Louisiana at Lafayette, \\ Lafayette, Louisiana, USA. \\  ${ }^{2}$ Department of Statistics, formerly of Presidency College, \\ Kolkata, India. \\}
\date{}
\maketitle
\vspace{-20pt}

\begin{abstract}
We study a simple statistic for benchmarking how well a sample preserves a known bivariate dependence structure. Given a target copula family (Clayton or Gumbel) and parameter $\theta_P$, the Copula Discrepancy (CD) compares the target Kendall’s tau $\tau(\theta_P)$ with the Kendall’s tau implied by a parameter $\hat\theta$ fitted to the sample within the target family, i.e., $|\tau(\theta_P)-\tau(\hat\theta)|$. We develop a moment-based version, prove consistency, asymptotic normality, and robustness results under i.i.d.\ sampling, and use an MLE-based version empirically for greater power against tail-structure misspecification. Building on this, we define two information-theoretic copula summaries, a copula KL divergence (CKL) and a copula entropy gap (CED), and establish basic consistency and central limit results for their plug-in estimators. In controlled experiments, CD reliably separates on-target and off-target copulas with matched Kendall’s $\tau$, provides a dependence-aware signal for tuning SGLD step sizes where Effective Sample Size favors overly aggressive (and biased) settings, and remains stably nonzero under deliberate tail-dependence mismatch where a naive $\tau$-based diagnostic fails; CKL and CED offer a complementary Shannon-style view that echoes these findings. Timing benchmarks show that both CD variants incur only millisecond-level overhead over the tested range and exhibit near-linear empirical scaling in sample size, providing a lightweight, dependence-focused complement to quadratic-cost omnibus discrepancies such as the Kernel Stein discrepancy (KSD).
\end{abstract}


\section{INTRODUCTION}

When the target copula family and parameter are fixed by design, practitioners often want a single, lightweight statistic that reports how well a sample preserves that target dependence, especially its tail behavior, without conflating this with marginal fit. This situation arises in controlled evaluations of approximate inference algorithms, stress tests of new samplers, and ablation studies in which the dependence structure is part of the model specification.

We propose the Copula Discrepancy (CD) to address this need. Fix a target copula family and parameter $\theta_P$. The CD compares the target’s Kendall’s tau map at $\theta_P$ with the same map evaluated at a parameter fitted to the sample within that target family. In practice, we use two versions: a moment-based estimator, which we analyze under i.i.d.\ sampling, and an MLE-based estimator, which we deploy empirically for greater power against tail-structure misspecification. Crucially, because these estimators rely on low-dimensional parameter fits, the CD scales close to linearly in sample size (up to optimizer iterations in the MLE case), offering a computational overhead orders of magnitude lower than kernel-based discrepancies. CD is designed to complement these global diagnostics: it isolates the fidelity of the dependence component relative to a chosen copula.

A core property of any multivariate distribution is its dependence structure. State-of-the-art diagnostics based on Stein’s method can, in principle, detect any distributional mismatch \citep{gorham2015measuring, gorham2017measuring}. However, their general-purpose nature does not distinguish between failures arising from the marginals and those arising from dependence. For a user asking, “Have I correctly captured the tail behavior between my parameters?”, this distinction matters. The standard Pearson correlation is well known to be inadequate in such settings: it can understate dependence in skewed or heavy-tailed data and provides no information about where in the distribution the dependence is strongest, which motivates the use of copula-based measures instead \citep{venter2002tails, embrechts2002correlation, katata2023technical}. In our experiments, CD often provides a clearer signal about dependence fidelity than generic summaries such as ESS, and in tail-mismatch settings it yields a more stable signal than KSD, so we view it as a targeted companion to these tools rather than a replacement.

Our construction is rooted in copula theory, which separates a multivariate distribution into its marginals and a copula that captures all of the dependence \citep{sklar1959}. CD focuses solely on this copula component. For clarity and tight control of ground truth, we work primarily with bivariate pseudo-observations and parametric Clayton and Gumbel families: some experiments use i.i.d.\ samples, while others (e.g., SGLD output) involve correlated draws.
Within this setting we prove consistency, asymptotic normality, and robustness properties for the moment-based CD, and we use the MLE-based CD and a pair of information-theoretic copula functionals, a copula KL divergence (CKL) and a copula entropy gap (CED), to give a complementary Shannon-style view. The exact computations extend mechanically to higher dimensions via pairwise aggregation or vine decompositions; a complete high-dimensional treatment is left to future work.

The rest of the paper is organized as follows. Section~\ref{sec:related} reviews related work on MCMC diagnostics and copula-based dependence measures. Section~\ref{sec:not} introduces notation and the pseudo-observation framework. Section~\ref{sec:cd_framework} formalizes the Copula Discrepancy, presents its moment-based and MLE-based estimators, and establishes key theoretical properties for the moment-based case. Section~\ref{sec:entropy} develops the Shannon-entropy and KL-based extensions (CKL and CED). Section~\ref{sec:experiments} reports three experiments that stress-test CD and its Shannon counterparts: separating copula families with matched Kendall’s $\tau$, tuning a biased SGLD sampler, and diagnosing tail-dependence mismatch. We then summarize computational overhead in a separate timing study (Figure~\ref{fig:overhead}). Finally, Section~\ref{sec:discussion} discusses practical guidance, computational costs, limitations, and broader impact.

\section{RELATED WORK}\label{sec:related}

Our work is situated at the intersection of three active research areas: modern MCMC diagnostics, the theory of biased MCMC methods, and the application of copula theory in statistics. We review each in turn.

\subsection{MCMC diagnostics and Stein's method}
The rise of biased samplers rendered early MCMC diagnostics insufficient. A significant breakthrough came with quality measures based on Stein's method \citep{stein72bound}, which can quantify the discrepancy between a sample and a target distribution without requiring a separate ground--truth sample. This approach was shown to be computationally practical, with the discrepancy attainable by solving a linear program \citep{gorham2015measuring}. This line of work was quickly extended to the Kernel Stein Discrepancy (KSD), which leverages reproducing kernels to define a closed-form quality measure computable by summing kernel evaluations across pairs of sample points \citep{gorham2017measuring, liu2016kernelized, chwialkowski2016kernel}. Because it is fundamentally a kernel $U$-statistic, a direct KSD evaluation has quadratic cost in the sample size. Critically, this work also demonstrated that the choice of kernel is crucial: some KSDs fail to detect non-convergence in moderate dimensions, whereas kernels with slowly decaying tails (for example, the IMQ kernel) can determine convergence under standard tail and score regularity conditions on the target \citep{gorham2017measuring}.

Our work does not challenge the power of these general-purpose tools. Instead, we draw inspiration from this framework to construct a diagnostic that is highly interpretable for dependence-specific questions. Other related works have focused on optimally thinning MCMC output \citep{riabiz2022optimal} or using Stein kernels for post-hoc correction of biased samples \citep{hodgkinson2020reproducing}. For a comprehensive review of MCMC diagnostics, see \citep{brooks1998general, vehtari2021rank}. Additionally, \citep{oates2017control} discusses the use of control functionals for improving the efficiency of Monte Carlo integration, which is relevant to the broader context of MCMC diagnostics.

\subsection{Biased MCMC methods}
Our work is directly motivated by the need to diagnose samples from the growing family of modern, scalable MCMC algorithms. SGLD \citep{welling2011bayesian} was a pioneering method, and its theoretical properties and bias-variance trade-offs have been rigorously analyzed \citep{teh2016consistency, vollmer2016exploration}. This family of methods has since expanded to include powerful samplers such as Stochastic Gradient Hamiltonian Monte Carlo (SGHMC), which requires a friction term to counteract gradient noise and maintain the correct target distribution \citep{chen2014sghmc}. These methods, which sacrifice exactness for speed, form the class of samplers for which our proposed diagnostic is most needed. For a detailed analysis of SGLD and its variants, see \citep{ma2015complete}.

\subsection{Copula theory and applications}
The theoretical engine of our work is Sklar's Theorem \citep{sklar1959}, which states that any multivariate joint distribution can be decomposed into its marginals and a copula that captures the entire dependence structure. For our practical implementation, we employ two standard estimation techniques for copula parameters: a fast method based on inverting Kendall's tau, and a more powerful method based on maximum likelihood estimation (MLE) \citep{choros_survey}.

Copulas are a mature tool for statistical modelling, for example, in Bayesian inference for handling mixed data types or constructing flexible dependence structures \citep{craiu2012mixed, panagiotelis2012pair}. However, their use as a targeted diagnostic for MCMC output appears limited in the published literature. Our contribution is to formalize and test such a diagnostic, keeping the marginals fixed by design and focusing solely on the copula component. For comprehensive overviews of copula theory, see \citep{nelsen2006introduction, joe2014dependence}.

\subsection{Dependence, tail behavior, and copula comparison}
A primary motivation for our work is the well-known inadequacy of linear Pearson correlation, particularly for non-elliptical or heavy-tailed distributions where it can be misleading or undefined \citep{embrechts2003modelling}. Capturing tail dependence, the behavior of variables during extreme events, is critical. Different copula families concentrate dependence in various parts of the distribution; for instance, the Gumbel copula is asymmetric and places more weight in the upper tail, while other families exhibit lower-tail or symmetric dependence \citep{venter2002tails}. This distinction allows for challenging test cases where two dependence structures share the exact rank correlation (such as Kendall's tau) but differ in their extreme-event behavior. For discussion of the limitations of linear correlation in capturing tail dependence, see \citep{embrechts2002correlation, mcneil2015quantitative}.

The literature on copula comparison and goodness-of-fit testing is rich, including smooth tests, minimum divergence methods, and quasi-likelihood approaches \citep{ngounou2024smooth, eguchi2025minimum, cambou2016quasi, gofcopula_journal}. These works typically use KL-type or entropy-based criteria to select or validate copula models. Our Shannon-style quantities (copula KL and entropy gap) are closest in spirit to this line of work. Still, we use them in a different role: as companions to CD for diagnosing dependence fidelity of MCMC output, rather than for copula model selection in observational data. From a theoretical perspective, our use of fitted copulas under misspecification follows the general quasi-MLE framework of \citet{white1982} and \citet{newey1994}, where estimators converge to KL-minimizing pseudo-true parameters with sandwich-type asymptotic variance. In our setting, this machinery underpins the pseudo-true copula parameter and explains why CD remains non-zero under structural mismatch.

\paragraph{When to use CD versus Stein/KSD.}
Stein-type discrepancies assess overall distributional fit and do not target a specific dependence structure. CD answers a narrower question: given a specified target copula, how well does a sample preserve that dependence, especially in the tails, independent of the marginals. In practice, the two tools are complementary: Stein/KSD can flag a discrepancy, and CD can help determine whether the copula is the source of the problem.


\section{NOTATION}\label{sec:not}

In this section, we introduce the notation used throughout the paper.

\paragraph{Distributions and samples.}
Let $P$ be the target probability distribution on the input space $\mathbb{R}^d$, and let $Q$ be an empirical distribution formed from a sample $\{x_i\}_{i=1}^n$ of size $n$. We use $Z \sim P$ and $X \sim Q$ to denote random variables. Expectations are written $\E[\cdot]$. A generic joint distribution is denoted by $H$ with marginals $F_i$. The Gelman--Rubin convergence diagnostic is denoted by $\widehat R$. We write $F_\delta = (1-\delta)F + \delta G$ for observation-level $\delta$–contamination, with $\delta \in [0,1]$. Throughout, our theoretical results focus on the bivariate case ($d=2$), while the notation itself is kept general.

\paragraph{Copula framework.}
A copula function $C$ on $[0,1]^d$ describes the dependence structure of a joint distribution. We consider a parametric copula family $\mathcal{C} = \{C_\theta : \theta \in \Theta\}$ indexed by a scalar dependence parameter $\theta$ from a compact parameter space $\Theta$. For Archimedean copulas, the family is defined by a generator function $\phi$. The true parameter for the target copula of $P$ is denoted $\theta_P$, while a parameter estimated from a sample is denoted $\hat{\theta}_n$ or $\hat{\theta}_Q$. When misspecification is present, the quasi-MLE or moment estimator converges to a pseudo-true parameter $\theta_Q^\star$ (the KL minimizer). The copula density corresponding to $C_\theta$ is written $c_\theta$. We write $C_P$ and $C_Q$ for the copulas associated with $P$ and $Q$, respectively.

\paragraph{Pseudo-observations and ties.}
We transform samples to $[0,1]^2$ using ranks to form pseudo-observations:
\[
u_i = \frac{R_i}{n+1}, \qquad v_i = \frac{S_i}{n+1},
\]
where $R_i, S_i$ are the \emph{average} (mid) ranks. This tie policy is appropriate for discrete or finite-precision outputs and is compatible with both Kendall's $\widehat\tau$ and copula MLE on pseudo-$U$'s. For fully continuous outputs (our default), this coincides with the ordinal ranking almost surely.

\paragraph{Estimators and discrepancies.}
Our primary quality measure is the Copula Discrepancy,
\[
\mathrm{CD}_n \;=\; \big|\tau(\theta_P) - \tau(\hat{\theta}_n)\big|.
\]
Here, $\tau$ is Kendall's rank correlation coefficient, and $\tau(\theta)$ is the map from a copula parameter to its implied Kendall's $\tau$ within the chosen family. The sample-based estimate of Kendall's $\tau$ is $\widehat{\tau}_n$. The moment-based estimator of the copula parameter is denoted $\hat{\theta}_n^{(M)}$, and the MLE-based estimator by $\hat{\theta}_n^{(\mathrm{MLE})}$. We write $\mathrm{CD}_n^{(M)}$ and $\mathrm{CD}_n^{(\mathrm{MLE})}$ for the corresponding discrepancies. 

\begin{remark}[Moment-based CD as a baseline]
Because the moment estimator is defined by inverting Kendall's tau, $\hat\theta^{(M)}_n=\tau^{-1}(\hat\tau_n)$, we have $\tau(\hat\theta^{(M)}_n)=\hat\tau_n$ and therefore
$CD^{(M)}_n = |\tau(\theta_P)-\hat\tau_n|$.
We treat $CD^{(M)}_n$ as a theoretically tractable baseline that anchors the more general MLE-based discrepancy, which compares fitted dependence through the copula likelihood and does not reduce to a single rank statistic.
\end{remark}

The asymptotic variance of $\sqrt{n}(\widehat{\tau}_n - \tau(\theta_P))$ is denoted $\sigma_\tau^2(\theta_P)$, and we use a standardized test statistic
\[
T_n^{(M)} \;=\; \frac{\sqrt{n}\,\mathrm{CD}_n^{(M)}}{\widehat{\sigma}_\tau}
\]
for hypothesis testing. The influence function of a functional $T$ at a distribution $F$ is written $\mathrm{IF}(z; T, F)$. Robustness properties are expressed via bounded influence and explicit $\varepsilon$–contamination bias bounds. For the SGLD hyperparameter study, the sampler step-size is denoted by $\epsilon$.

\paragraph{Shannon-style copula functionals.}
For a copula density $c$ with copula $C$, the Shannon differential entropy is
\[
H(C) \;=\; - \int_{[0,1]^2} c(u,v)\,\log c(u,v)\,du\,dv,
\]
and for two copulas with densities $c_P$ and $c_Q$ we write the copula Kullback--Leibler divergence as
\[
D_{\mathrm{KL}}(C_A \,\|\, C_B) \;=\; \int_{[0,1]^2} c_A(u,v)\,\log\frac{c_A(u,v)}{c_B(u,v)}\,du\,dv.
\]
Given a target copula $C_{\theta_P}$ and a fitted copula $C_{\hat{\theta}_Q}$ within the same family, we define the copula entropy gap
\[
\mathrm{CED}(Q;P) \;=\; \big|H(C_{\hat{\theta}_Q}) - H(C_{\theta_P})\big|
\]
and the copula KL discrepancy
\[
\mathrm{CKL}(Q;P) \;=\; D_{\mathrm{KL}}\!\big(C_{\hat{\theta}_Q} \,\|\, C_{\theta_P}\big).
\]

Their empirical counterparts, $\widehat{\mathrm{CED}}_n$ and $\widehat{\mathrm{CKL}}_n$, are obtained by plug-in and sample averages using pseudo-observations, with a precomputed entropy constant for the target copula.

\paragraph{Other diagnostics.}
The Kernel Stein Discrepancy (KSD) between a sample and the target is denoted generically by $\mathrm{KSD}$; in our experiments it is computed with an inverse multiquadric kernel
\[
k(x,y) = (c^2 + \|x-y\|^2)^\beta, \qquad c = 1,\ \beta = -\tfrac{1}{2}.
\]
Effective Sample Size is denoted by $\mathrm{ESS}$; in the SGLD experiments we use the minimum univariate ESS across coordinates for each replication and then report the mean and its 95\% confidence interval across replications.



\section{THE COPULA DISCREPANCY FRAMEWORK} \label{sec:cd_framework}

In this section, we introduce our quantity of interest, the Copula Discrepancy (CD). We first motivate the need for a quality measure focused specifically on dependence structure, then leverage the theory of copulas to formalize this into a computable discrepancy, and finally establish its key theoretical properties.

\paragraph{Scope of results.}
All theory here is for i.i.d.\ pseudo-observations in the bivariate case ($d=2$). The target family (Clayton or Gumbel) and its parameter $\theta_P$ are known by design. Our asymptotics are for the moment-based estimator; the MLE variant is used empirically but not analyzed here.

\subsection{Motivation and Formal Definition}

Our primary goal is to quantify the discrepancy between a sample distribution $Q$ and a target distribution $P$ in a manner that is (i) specifically sensitive to distortions in the dependence structure, (ii) computationally feasible, and (iii) capable of distinguishing between high- and low-quality samples. While other powerful diagnostics exist, they are not explicitly designed to assess the fidelity of the dependence model. A sampler may approximate the marginals of a target well while failing to capture crucial tail dependencies, a shortcoming that could be missed by existing methods.

To address this gap, we turn to Sklar's Theorem \citep{sklar1959}, which states that any $d$-dimensional joint distribution $H$ with continuous marginals $F_1, \dots, F_d$ can be uniquely decomposed into its marginals and a copula function $C$:
\begin{equation}
    H(x_1, \dots, x_d) = C(F_1(x_1), \dots, F_d(x_d)).
\end{equation}
The copula $C$ fully encapsulates the dependence properties of $H$. We posit that for a sample $Q$ to be a high-quality approximation of $P$, its underlying copula, $C_Q$, must be a close match to the target's copula, $C_P$. Our goal is thus to measure a discrepancy $d(C_P, C_Q)$. To formalize this, we begin with the necessary definitions and conditions.

\begin{definition}[Archimedean copulas]\label{def:copula_family}
Let $\mathcal{C}=\{C_\theta:\theta\in\Theta\subset\mathbb{R}\}$ be a bivariate Archimedean family with generator $\phi_\theta:[0,1]\to[0,\infty]$ such that: (i) $\phi_\theta$ is strictly decreasing and convex; (ii) $\phi_\theta(1)=0$ and $\phi_\theta(0)=\infty$; and (iii) $C_\theta(u,v)=\phi_\theta^{-1}\!\big(\phi_\theta(u)+\phi_\theta(v)\big)$.
\end{definition}

\begin{assumption}[Regularity conditions]\label{assump:regularity}
We restrict to the bivariate \textbf{Clayton} ($\theta>0$) and \textbf{Gumbel} ($\theta\ge 1$) copula families. Let $\Theta\subset\mathbb{R}$ be a compact interval contained in the admissible parameter range of the family under consideration and chosen to cover the parameter values used in our experiments. Let $\theta_P\in\mathrm{int}(\Theta)$ denote the target parameter. On $\Theta$:
\begin{enumerate}
\item[(i)] (\textit{Identifiability}) The map $\theta\mapsto C_\theta$ is injective.
\item[(ii)] (\textit{Smoothness in $\theta$}) For each $(u,v)\in(0,1)^2$, the map $\theta\mapsto C_\theta(u,v)$ is twice continuously differentiable on $\mathrm{int}(\Theta)$.
\item[(iii)] (\textit{Stable moment inversion}) Kendall's $\tau$ map $\tau:\Theta\to\tau(\Theta)$ is strictly increasing and continuously differentiable on $\mathrm{int}(\Theta)$, with derivative bounded away from zero on $\Theta$, i.e.
\[
\inf_{\theta\in\Theta}\lvert\tau'(\theta)\rvert>0.
\]
In particular, for Clayton $\tau(\theta)=\theta/(\theta+2)$ with $\tau'(\theta)=2/(\theta+2)^2$, and for Gumbel $\tau(\theta)=1-1/\theta$ with $\tau'(\theta)=1/\theta^2$. Since $\Theta$ is compact and contained in the admissible range, the above infimum is strictly positive.
\item[(iv)] (\textit{Density regularity}) The copula density $c_\theta(u,v)$ exists and is continuous in $(u,v)$ and in $\theta$ on $(0,1)^2\times \mathrm{int}(\Theta)$.
\end{enumerate}
These conditions ensure identifiability and well-posed moment inversion $\hat\theta=\tau^{-1}(\hat\tau)$; see, e.g., \citet{nelsen2006introduction,joe2014dependence,genest1995semiparametric}.
\end{assumption}

\noindent\emph{Remark.} In all experiments, the target parameter satisfies $\theta_P\in\mathrm{int}(\Theta)$. Moreover, we choose $\Theta$ so that the values of $\theta$ encountered in our simulations remain away from the boundary, which avoids boundary effects when using $\hat\theta=\tau^{-1}(\hat\tau)$ and in subsequent asymptotic expansions.

We define the Copula Discrepancy in the space of Kendall's tau, yielding a universally interpretable measure of concordance that is invariant under strictly monotone transformations of the marginals.

\begin{definition}[Copula Discrepancy]\label{def:cd}
Let $P$ denote a target distribution with copula $C_{\theta_P}$ in the chosen parametric family, and let $Q_n$ denote an empirical distribution (e.g., produced by an approximate inference algorithm). Let $\hat{\theta}_n$ be the parameter obtained by fitting the \emph{target} copula family to $Q_n$ (e.g., by moments or MLE). The Copula Discrepancy is
\[
\mathrm{CD}_n \;=\; \left| \tau(\theta_P) - \tau(\hat{\theta}_n) \right|.
\]
\end{definition}

\vspace{0.5em}
\noindent\fbox{%
    \parbox{0.95\linewidth}{%
        \textbf{Definitions at a Glance}
        \begin{itemize}
            \item \textbf{Target $P$:} a specific parametric copula family (e.g., Gumbel) with known parameter $\theta_P$.
            \item \textbf{Sample $Q_n$:} the empirical distribution from the approximate inference algorithm being benchmarked.
            \item \textbf{Fitted parameter $\hat{\theta}_n$:} the parameter obtained by fitting the \textit{target} family to $Q_n$ (via Moments or MLE).
            \item \textbf{Statistic:} $\mathrm{CD}_n = \big| \tau(\theta_P) - \tau(\hat{\theta}_n)\big|$.
        \end{itemize}
    }%
}
\vspace{0.5em}

\begin{remark}[Intuition under misspecification]\label{rem:pseudotrue}
When $Q_n$ is generated from a distribution whose copula does \emph{not} belong to the target family, standard estimation procedures (e.g., MLE) converge to a \emph{pseudo-true} parameter $\theta_Q^\star\in\Theta$ that minimizes the Kullback--Leibler divergence within the target family \citep{white1982,newey1994}. In that case, $\mathrm{CD}_n$ converges (typically) to the nonnegative limit
\[
\left|\tau(\theta_P)-\tau(\theta_Q^\star)\right|,
\]
which equals $0$ only if the fitted target family can match the target concordance level. Thus, $\mathrm{CD}_n$ quantifies mismatch in dependence \emph{as measured on the target tau scale}, even when the data do not come from the target family.
\end{remark}

\subsection{Estimation Methods}

The practical computation of the CD requires estimating the sample parameter $\hat{\theta}_Q$ from pseudo-observations $\{u_i\}_{i=1}^n$ obtained by transforming the original sample using the empirical CDF. We consider two primary methods, summarized in Algorithms~\ref{alg:cd_moment} and \ref{alg:cd_mle} which are provided in Appendix~\ref{app:1} 

\begin{itemize}[leftmargin=*]
    \item \textit{Moment-based (Algorithm~\ref{alg:cd_moment}):} A fast estimation is achieved by first calculating the empirical Kendall's tau, $\hat{\tau}_Q$, and then inverting the known relationship for the family $\mathcal{C}$ to get $\hat{\theta}_Q = \tau^{-1}(\hat{\tau}_Q)$. This method is computationally efficient and ideal for iterative tasks like hyperparameter tuning.
    \item \textit{Maximum Likelihood (MLE) (Algorithm~\ref{alg:cd_mle}):} A more robust method is to find the parameter $\hat{\theta}_Q$ that maximizes the log-likelihood of the pseudo-observations: $\hat{\theta}_Q = \argmax_{\theta \in \Theta} \sum_{i=1}^n \log c_\theta(u_i)$. This method is more powerful for detecting subtle structural mismatches. While we develop the explicit theory for the tractable moment-based case, the standard asymptotic properties of consistency and normality for the MLE-based estimator are well-established under similar regularity conditions (\citep{nelsen2006introduction, joe2014dependence}).
\end{itemize}

\subsection{Statistical Properties of the Moment-Based Estimator}

We now establish key statistical properties for the moment-based estimator $\hat{\theta}_n^{(M)}=\tau^{-1}(\hat{\tau}_n)$, and the resulting discrepancy $\mathrm{CD}_n^{(M)}$.

\begin{theorem}[Consistency of Moment-Based Estimator]\label{thm:consistency_moment}
Let $\{(U_i,V_i)\}_{i=1}^n$ be i.i.d.\ pseudo-observations from a copula $C_{\theta_P}$ in the target family. Let $\hat{\tau}_n$ be the sample Kendall's tau and $\hat{\theta}_n^{(M)}=\tau^{-1}(\hat{\tau}_n)$ be the moment-based estimator. Under Assumption~\ref{assump:regularity}:
\begin{enumerate}
    \item $\hat{\tau}_n \xrightarrow{p} \tau(\theta_P)$ as $n\to\infty$.
    \item $\hat{\theta}_n^{(M)} \xrightarrow{p} \theta_P$ as $n\to\infty$.
    \item $\mathrm{CD}_n^{(M)} \xrightarrow{p} 0$ as $n\to\infty$.
\end{enumerate}
\end{theorem}

\noindent Full proof is given in Appendix~\ref{app:2}.

\begin{theorem}[Asymptotic Distribution of Moment-Based CD]\label{thm:asymptotic_moment}
Under Assumption~\ref{assump:regularity}, as $n\to\infty$,
\[
\sqrt{n}\,\mathrm{CD}_n^{(M)} \xrightarrow{d} \left|N\!\left(0,\sigma_\tau^2(\theta_P)\right)\right|,
\]
where $\sigma_\tau^2(\theta_P)$ is the asymptotic variance in
$\sqrt{n}\big(\hat{\tau}_n-\tau(\theta_P)\big)\xrightarrow{d}N\!\left(0,\sigma_\tau^2(\theta_P)\right)$.
\end{theorem}

\noindent Full proof is given in Appendix~\ref{app:2}.

\begin{corollary}[Asymptotic level-$\alpha$ test based on moment-based CD]\label{cor:test_moment_cd}
Under the conditions of Theorem~\ref{thm:asymptotic_moment}, define the test statistic
\[
T_n^{(M)}:=\frac{\sqrt{n}\,\mathrm{CD}_n^{(M)}}{\sigma_\tau(\theta_P)}.
\]
Then $T_n^{(M)}\xrightarrow{d}|N(0,1)|$. Consequently, an asymptotic level-$\alpha$ test that rejects the null hypothesis $H_0:\theta=\theta_P$ (equivalently, $\tau(\theta)=\tau(\theta_P)$) rejects whenever
\[
T_n^{(M)} > z_{1-\alpha/2},
\]
where $z_{1-\alpha/2}$ is the $(1-\alpha/2)$ quantile of the standard normal distribution.
\end{corollary}

\noindent Full proof is given in Appendix~\ref{app:2}.

\subsection{Hypothesis Testing and Robustness}

Building on the asymptotic properties, we can formalize a hypothesis test and analyze the estimator's robustness.

\begin{definition}[Copula equivalence test]\label{def:copula_test}
Let $Q$ denote the (unknown) copula of the data-generating distribution for $\{(U_i,V_i)\}_{i=1}^n$ and let $\theta_Q$ be the parameter of the chosen target family that indexes $Q$ under correct specification. We test
\begin{align*}
H_0 &: \theta_Q=\theta_P \quad \text{(equivalently, } \tau(\theta_Q)=\tau(\theta_P)\text{)},\\
H_1 &: \theta_Q\neq\theta_P \quad \text{(equivalently, } \tau(\theta_Q)\neq\tau(\theta_P)\text{)}.
\end{align*}
For the moment-based estimator, define $\mathrm{CD}_n^{(M)}:=|\tau(\theta_P)-\hat\tau_n|$ and the standardized statistic
\[
T_n^{(M)}:=\frac{\sqrt{n}\,\mathrm{CD}_n^{(M)}}{\hat\sigma_\tau},
\]
where $\hat\sigma_\tau$ is any estimator satisfying $\hat\sigma_\tau\xrightarrow{p}\sigma_\tau(\theta_P)$ under $H_0$.
\end{definition}

\begin{theorem}[Asymptotic test and rejection rule]\label{thm:test}
Under $H_0$ and Assumption~\ref{assump:regularity}, if $\hat\sigma_\tau\xrightarrow{p}\sigma_\tau(\theta_P)$, then
\[
T_n^{(M)} \xrightarrow{d} |N(0,1)|.
\]
Consequently, an asymptotic level-$\alpha$ test rejects $H_0$ whenever
\[
T_n^{(M)} > z_{1-\alpha/2},
\]
where $z_{1-\alpha/2}$ is the $(1-\alpha/2)$ quantile of the standard normal distribution.
\end{theorem}

Full proof is given in Appendix~\ref{app:2}.

Finally, we analyze robustness through a bounded-influence result (Theorem~\ref{thm:cd_bounded_if}) and an explicit $\varepsilon$–contamination bias bound (Theorem~\ref{thm:cd_contam}).

\begin{theorem}[Bounded influence for the moment-based CD] \label{thm:cd_bounded_if}
Let $F_{\theta_P}$ denote the true model for the pseudo-observations $Z=(U,V)$ and define
\[
\mathrm{CD}^{(M)}(F)=\big|\tau(\theta_P)-\tau(F)\big|,
\]
where $\tau(F)$ is Kendall's $\tau$ functional under distribution $F$.
Then the influence function of $\mathrm{CD}^{(M)}$ at $F_{\theta_P}$ exists in the subgradient sense and satisfies
\[
\mathrm{IF}(z;\,\mathrm{CD}^{(M)},F_{\theta_P}) \;\in\; \xi\cdot \mathrm{IF}(z;\,\tau,F_{\theta_P})
\quad\text{for some }\xi\in[-1,1].
\]
Moreover, Kendall's $\tau$ has a bounded influence function with
\[
\sup_{z}\big|\mathrm{IF}(z;\,\tau,F_{\theta_P})\big|\le 4,
\]
and hence
\[
\sup_{z}\big|\mathrm{IF}(z;\,\mathrm{CD}^{(M)},F_{\theta_P})\big|\le 4.
\]
Therefore $\mathrm{CD}^{(M)}$ is B-robust.
\end{theorem}

Full proof is given in Appendix~\ref{app:2}.

\begin{theorem}[Contamination stability of the moment-based CD] \label{thm:cd_contam}
Let $F_\delta=(1-\delta)F+\delta G$ denote $\delta$-contamination at the observation level with $\delta\in[0,1]$. Then
\[
\big|\tau(F_\delta)-\tau(F)\big| \;\le\; 4\delta-2\delta^2 \;\le\; 4\delta,
\]
and consequently
\[
\mathrm{CD}^{(M)}(F_\delta) \;\le\; \big|\tau(\theta_P)-\tau(F)\big| \;+\; 4\delta-2\delta^2.
\]
In particular, under $H_0$ where $\tau(\theta_P)=\tau(F)$, one has $\mathrm{CD}^{(M)}(F_\delta)\le 4\delta-2\delta^2$.
\end{theorem}

Full proof is given in Appendix~\ref{app:2}.

\begin{remark}[Why no breakdown-point number]
Because $\tau\in[-1,1]$, $\mathrm{CD}^{(M)}\in[0,2]$ is bounded. Classical Hampel breakdown (based on unbounded explosion) is therefore not informative here.
Theorem~\ref{thm:cd_bounded_if} and Theorem~\ref{thm:cd_contam} provide the relevant robustness picture: bounded influence and an explicit $\varepsilon$-contamination bias bound.
\end{remark}


\section{INFORMATION-THEORETIC EXTENSION: COPULA \\SHANNON ENTROPY AND KL}
\label{sec:entropy}

We now formalize an information-theoretic view that complements CD by scoring \emph{copula densities} directly via Shannon entropy and Kullback--Leibler (KL) divergence. Throughout, $C_\theta$ denotes a bivariate copula with density $c_\theta$ on $[0,1]^2$.

\subsection{Definitions}
\begin{definition}[Copula Shannon entropy]
\label{def:copula_entropy}
The Shannon differential entropy of a copula density $c$ is
\[
H(C)\;=\;-\!\!\int_{[0,1]^2} c(u,v)\,\log c(u,v)\,du\,dv.
\]
\end{definition}

\begin{definition}[Copula KL divergence]
\label{def:copula_kl}
For two copulas with densities $c_A$ and $c_B$,
\[
D_{\mathrm{KL}}(C_A\,\|\,C_B)\;=\;\int_{[0,1]^2} c_A(u,v)\,\log\!\frac{c_A(u,v)}{c_B(u,v)}\,du\,dv.
\]

\end{definition}

\begin{remark}[Invariance to monotone marginal transforms]
\label{rem:invariance}
If $(X,Y)$ has joint $H$ with continuous marginals and copula $C$, and $(f,g)$ are strictly increasing, then $(f(X),g(Y))$ has the same copula $C$. Hence $H(C)$ and $D_{\mathrm{KL}}(C_A\|C_B)$ depend only on the \emph{dependence} and are marginal-free.
\end{remark}

\subsection{Basic properties}
\begin{proposition}[Mutual information equals copula negentropy]
\label{prop:mi_equals_negentropy}
Let $(X,Y)$ have continuous marginals and copula density $c$. Then the mutual information satisfies
\[
I(X;Y)\;=\;\int_{[0,1]^2} c(u,v)\,\log c(u,v)\,du\,dv\;=\;-\,H(C).
\]
\end{proposition}
Proof appears in Appendix~\ref{app:2}.

\begin{proposition}[Nonnegativity and identifiability]
\label{prop:kl_nonneg}
$D_{\mathrm{KL}}(C_P\|C_Q)\ge 0$, with equality iff $c_P(u,v)=c_Q(u,v)$ almost everywhere on $[0,1]^2$.
\end{proposition}
Proof appears in Appendix~\ref{app:2}.

\subsection{An entropy/KL-based copula discrepancy}
Fix a target family $\{C_\theta\}$ and parameter $\theta_P$. Given a sample distribution $Q$ (obtained from an approximate inference method), let $\hat\theta_Q$ be any fitted parameter within the \emph{target} family (moment or MLE). We define two complementary scores:
\begin{align}
\mathrm{CED}(Q;P) &:= \big|\,H(C_{\hat\theta_Q})-H(C_{\theta_P})\,\big|, \label{eq:ced}\\
\mathrm{CKL}(Q;P) &:= D_{\mathrm{KL}}\!\left(C_{\hat{\theta}_Q}\,\Vert\,C_{\theta_P}\right).\label{eq:ckl}
\end{align}
Both scores are invariant to monotone marginal transforms. Moreover, $\mathrm{CKL}(Q;P)=0$ iff $C_{\hat\theta_Q}=C_{\theta_P}$ (equivalently $\hat\theta_Q=\theta_P$ by identifiability). The CED is an entropy-gap diagnostic that is also $0$ when $H(C_{\hat\theta_Q})=H(C_{\theta_P})$.

The complete computational procedure is summarized in Algorithm~\ref{alg:ckl_ced} (Appendix~\ref{app:1}).

\subsection{Estimation from pseudo-observations}
A unified, step-by-step implementation for estimating $\widehat{\mathrm{CKL}}$ and $\widehat{\mathrm{CED}}$ is given in Algorithm~\ref{alg:ckl_ced} (Appendix~\ref{app:1}).

Let $\{(U_i,V_i)\}_{i=1}^n$ be pseudo-observations on $[0,1]^2$ and let $\hat\theta_Q$ be fitted within the target family (moment or MLE).

\paragraph{Plug-in for CED.}
When the family admits a closed-form entropy $H(C_\theta)$ (e.g., Clayton, Gumbel), compute
\[
\widehat{\mathrm{CED}}_n \;=\; \big|\,H(C_{\hat\theta_Q})-H(C_{\theta_P})\,\big|.
\]
If $H(C_\theta)$ has no closed form, approximate $H(C_\theta)=-\mathbb{E}_{C_\theta}[\log c_\theta(U,V)]$ by Monte Carlo under $C_\theta$.

\paragraph{Monte Carlo estimator for CKL (matches Definition~\ref{def:copula_kl}).}
By definition,
\[
D_{\mathrm{KL}}(C_{\hat{\theta}_Q}\Vert C_{\theta_P})
=
\mathbb{E}_{(U,V)\sim C_{\hat{\theta}_Q}}
\!\left[\log c_{\hat{\theta}_Q}(U,V)-\log c_{\theta_P}(U,V)\right].
\]
Accordingly, draw an \emph{independent} Monte Carlo sample $\{(\widetilde U_j,\widetilde V_j)\}_{j=1}^m\overset{\text{i.i.d.}}{\sim}C_{\hat{\theta}_Q}$ and compute
\[
\widehat{\mathrm{CKL}}_{m} \;=\; \frac{1}{m}\sum_{j=1}^m \left[ \log c_{\hat{\theta}_Q}(\widetilde U_j,\widetilde V_j) - \log c_{\theta_P}(\widetilde U_j,\widetilde V_j) \right].
\]
This estimator is a direct plug-in Monte Carlo approximation to the defining expectation of KL.

\subsection{Asymptotics (i.i.d.\ bivariate pseudo-observations)}
\begin{assumption}[Regularity for entropy/KL]
\label{assump:entropy_reg}
On a compact $\Theta$, $c_\theta(u,v)>0$ and is twice continuously differentiable in $(u,v,\theta)$ on $(0,1)^2\times\Theta$; $\sup_{\theta\in\Theta}\int_{[0,1]^2} c_\theta|\log c_\theta|\,du\,dv<\infty$; and $\theta\mapsto c_\theta$ is identifiable.
\end{assumption}

\begin{theorem}[Consistency of entropy/KL plug-ins]
\label{thm:entropy_consistency}
Assume $\hat\theta_Q\xrightarrow{p}\theta_Q^\star$ (the pseudo-true limit under $Q$) and Assumption~\ref{assump:entropy_reg} holds. Then
\[
\widehat{\mathrm{CED}}_n \xrightarrow{p} \big|\,H(C_{\theta_Q^\star})-H(C_{\theta_P})\,\big|.
\]
Moreover, the KL plug-in satisfies
\[
D_{\mathrm{KL}}(C_{\hat\theta_Q}\,\Vert\,C_{\theta_P}) \xrightarrow{p} D_{\mathrm{KL}}(C_{\theta_Q^\star}\,\Vert\,C_{\theta_P}),
\]
and if $\widehat{\mathrm{CKL}}_{m}$ is computed by Monte Carlo with $m\to\infty$, then
\[
\widehat{\mathrm{CKL}}_{m} \xrightarrow{p} D_{\mathrm{KL}}(C_{\theta_Q^\star}\,\Vert\,C_{\theta_P}).
\]
\end{theorem}
A detailed proof is given in Appendix~\ref{app:2}.

\begin{theorem}[Asymptotic normality for CKL]
\label{thm:ckl_clt}
Under Assumption~\ref{assump:entropy_reg} and standard smoothness for $\hat\theta_Q$, suppose
\[
\sqrt{n}\,(\hat\theta_Q-\theta_Q^\star)\xrightarrow{d}N(0,\Sigma_Q).
\]
Define $\psi(\theta)=D_{\mathrm{KL}}(C_{\theta}\,\Vert\,C_{\theta_P})$. If $\psi$ is differentiable at $\theta_Q^\star$, then
\[
\sqrt{n}\Big(\psi(\hat\theta_Q)-\psi(\theta_Q^\star)\Big)\xrightarrow{d}
N\!\left(0,\ \psi'(\theta_Q^\star)\,\Sigma_Q\,\psi'(\theta_Q^\star)^\top\right).
\]
If in addition $\widehat{\mathrm{CKL}}_{m}$ is computed by Monte Carlo with $m=m_n\to\infty$ and $m_n/n\to\infty$, then
\[
\sqrt{n}\Big(\widehat{\mathrm{CKL}}_{m_n}-\psi(\theta_Q^\star)\Big)\xrightarrow{d}
N\!\left(0,\ \psi'(\theta_Q^\star)\,\Sigma_Q\,\psi'(\theta_Q^\star)^\top\right).
\]
\end{theorem}
A detailed proof is given in Appendix~\ref{app:2}.

\begin{remark}[Scope of Theorem 7]
Theorem~7 characterizes the $\sqrt{n}$ fluctuation of $\psi(\hat\theta_Q)$ (and hence $\widehat{CKL}$) around the pseudo-true limit $\theta_Q^\star$, the KL minimizer under possible misspecification. Under correct specification with $\theta_Q^\star=\theta_P$, $\psi(\theta)=D_{KL}(C_\theta\|C_{\theta_P})$ attains its minimum at $\theta_P$ and $\psi'(\theta_P)=0$, so the $\sqrt{n}$ limit in Theorem~7 becomes degenerate. Accordingly, we do not use Theorem~7 to justify a Wald-type test of the boundary null $H_0: CKL(Q;P)=0$; that case requires a different scaling/limit theory.
\end{remark}

\paragraph{Practical use.}
We follow a simple workflow. For each run we: (i) form pseudo-observations; (ii) fit $\hat\theta_Q$ by MLE within the target family (with an option to switch to moments for speed); (iii) compute $\widehat{\mathrm{CED}}_n$ as an entropy plug-in; and (iv) compute $\widehat{\mathrm{CKL}}_{m}$ by Monte Carlo draws from $C_{\hat\theta_Q}$ using a fixed $m$ (we use $m=10^5$ with a fixed random seed for reproducibility).

\section{EXPERIMENTS}
\label{sec:experiments}
We now put the Copula Discrepancy (CD) to the test. Consistent with our framework's focus on benchmarking, we work in controlled settings where the ``right answer", the target copula family and its parameter, is known by design. This allows us to rigorously evaluate whether the CD can succeed where other metrics fail. We structure our evaluation around three specific questions. First, can the CD distinguish between different dependence families, even when they share the exact rank correlation? Second, does it provide a better signal for hyperparameter tuning in biased MCMC than standard diagnostics like Effective Sample Size? And third, can it robustly detect subtle mismatches in tail behavior that naive summary statistics miss? We address each question in turn.

\subsection{Verifying Sensitivity to Dependence Structure}\label{sec:exp1}

Our first experiment tests whether the Copula Discrepancy (CD) can distinguish among different dependence structures in a setting where simpler diagnostics would struggle.

\textbf{Setup.} We focus on the MLE-based CD (Algorithm~\ref{alg:cd_mle}), since it works directly with the copula likelihood rather than just rank correlation. The target distribution is a Gumbel copula, and the off-target distribution is a Clayton copula. The parameters are chosen so that both copulas share the same population Kendall’s tau, $\tau = 0.6$ ($\theta_{\text{Gumbel}} = 2.5$, $\theta_{\text{Clayton}} = 3.0$). In every replication we draw one sample from the Gumbel target and one from the Clayton copula, fit a Gumbel copula to each by MLE, and compute the resulting CD. In addition to CD, we also record two information–theoretic summaries: a copula KL-based discrepancy (CKL), which measures the divergence of the fitted copula from a Kendall’s-tau-matched Gumbel reference in KL units (computed in the forward direction $\mathrm{KL}(\widehat{C}\,\|\,C_{\mathrm{ref}})$), and a copula entropy gap (CED), which compares the Shannon entropy of the fitted copula to that of the same Gumbel reference.

\begin{figure}[ht]
    \centering
    \includegraphics[width=\linewidth]{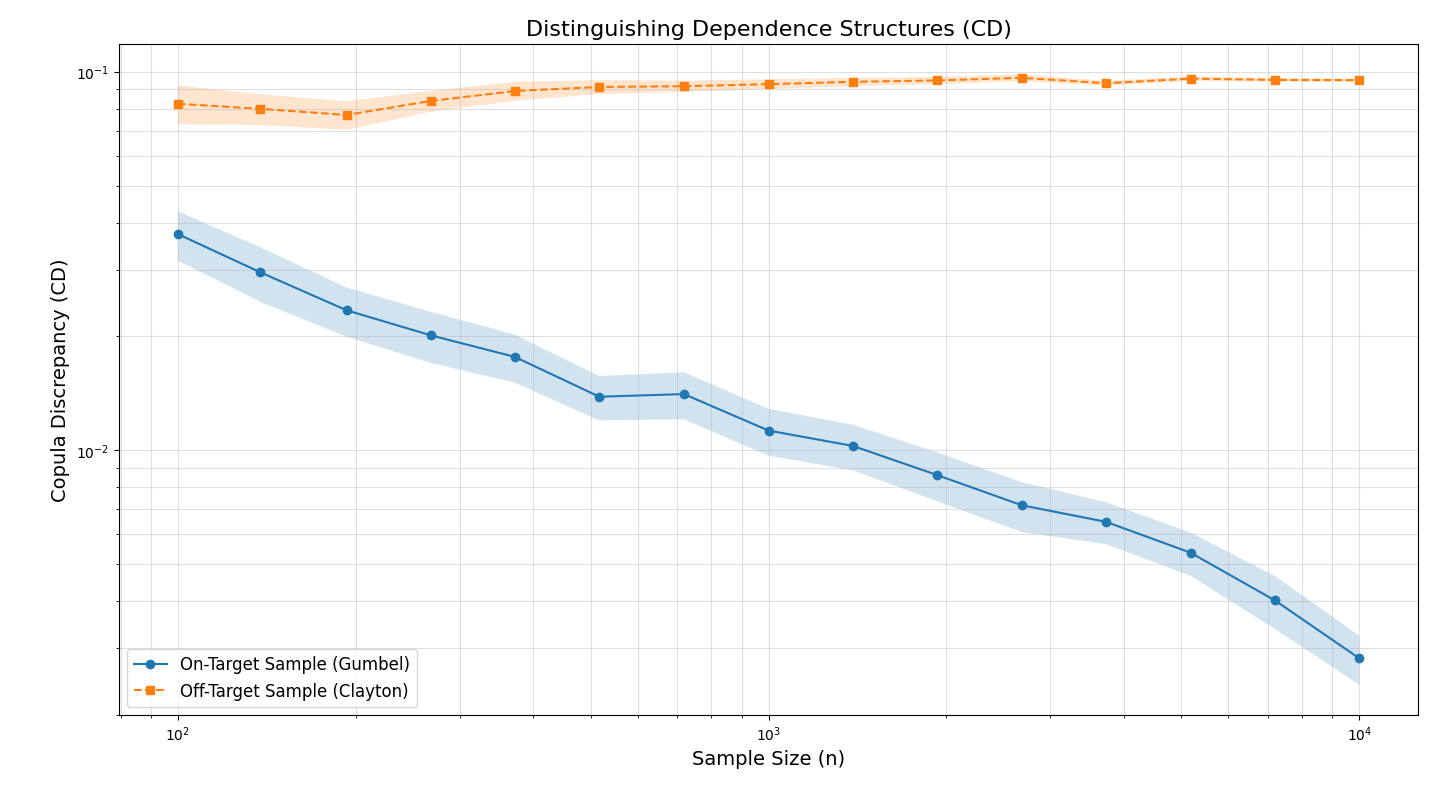}
    \caption{MLE-based Copula Discrepancy (CD) for samples drawn from the on-target Gumbel copula (blue) and the off-target Clayton copula (orange), both constructed to have the same population Kendall’s tau $\tau = 0.6$. Each point is the mean CD over 100 replications; the shaded bands show the 95\% confidence interval for the mean. Even though rank correlation is matched, the on-target CD decays rapidly toward zero as $n$ increases, while the off-target CD remains substantially larger (around $8\times 10^{-2}$ to $10^{-1}$) and varies only mildly across sample sizes.}
    \label{fig:exp1_cd}
\end{figure}

\begin{figure}[ht]
    \centering
    \includegraphics[width=\linewidth]{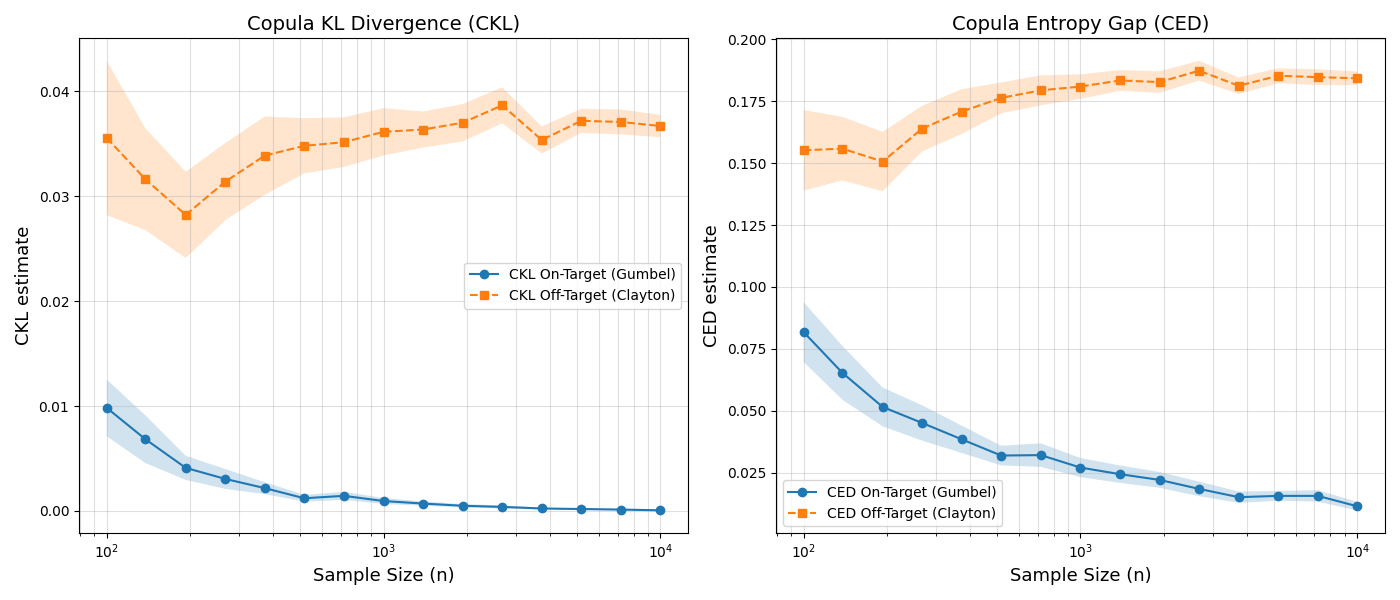}
    \caption{Information–theoretic diagnostics for Experiment~\ref{sec:exp1}. \textbf{Left:} Copula KL-based discrepancy (CKL) evaluated against a Kendall’s-tau-matched Gumbel reference copula. \textbf{Right:} Copula entropy gap (CED) between the fitted copula and the Gumbel target. In both panels, blue corresponds to on-target Gumbel samples and orange to off-target Clayton samples; lines show means over 100 replications and shaded bands show 95\% confidence intervals. CKL and CED reproduce the same qualitative ordering as CD: fits to on-target samples trend toward lower discrepancy as $n$ grows, while fits to off-target samples remain consistently higher.}
    \label{fig:exp1_shannon}
\end{figure}

\textbf{Result and analysis.} Figure~\ref{fig:exp1_cd} shows that CD clearly separates the two cases. For Gumbel samples (on target), the mean CD starts at $3.73\times 10^{-2}$ at $n=100$ and decreases steadily as $n$ grows, reaching $2.82\times 10^{-3}$ at $n=10{,}000$. This sharp decay is exactly what we expect when the fitted family matches the data-generating dependence. In contrast, for Clayton samples (off target), the mean CD is substantially larger and remains near-constant on the scale of the plot: it starts at $8.24\times 10^{-2}$ at $n=100$ and is $9.52\times 10^{-2}$ at $n=10{,}000$, staying in the range $[7.7\times 10^{-2},\,9.7\times 10^{-2}]$ across sample sizes. Across all sample sizes, the blue and orange curves remain well separated and their 95\% confidence bands do not overlap, indicating that CD robustly detects misspecification even when Kendall’s tau is perfectly matched by design.

The Shannon-style diagnostics in Figure~\ref{fig:exp1_shannon} tell the same story from a different angle. The CKL curves exhibit a persistent separation: on-target fits decrease rapidly toward zero as $n$ increases (from $9.83\times 10^{-3}$ at $n=100$ to $6.1\times 10^{-5}$ at $n=10{,}000$), whereas off-target fits remain much larger, staying on the order of $3\times 10^{-2}$ to $4\times 10^{-2}$ across sample sizes (e.g., $3.56\times 10^{-2}$ at $n=100$ and $3.67\times 10^{-2}$ at $n=10{,}000$). The entropy gap CED behaves similarly: for on-target data the gap shrinks from about $8.18\times 10^{-2}$ at $n=100$ to $1.14\times 10^{-2}$ at $n=10{,}000$, while for off-target data it remains large and stable on the scale of the plot (roughly $1.5\times 10^{-1}$ to $1.9\times 10^{-1}$, with $1.55\times 10^{-1}$ at $n=100$ and $1.84\times 10^{-1}$ at $n=10{,}000$). Taken together, these results confirm that the MLE-based CD and its Shannon counterparts are sensitive to the underlying dependence structure and can reliably flag misspecified copula families even when rank correlation is held fixed. Detailed numerical values, including the mean and 95\% confidence interval at each sample size, are reported in Appendix~\ref{app:3}.

\subsection{Hyperparameter Selection Case Study}\label{sec:exp2}
We next test the CD in a practical setting where standard diagnostics often give misleading guidance: tuning the step-size $\epsilon$ of a Stochastic Gradient Langevin Dynamics (SGLD) sampler targeting a bimodal Gaussian mixture posterior \citep{gorham2015measuring}.

\textbf{Setup.} The SGLD step-size controls a familiar trade-off: very small $\epsilon$ produces a slowly mixing but nearly unbiased chain, while larger values mix quickly at the cost of substantial asymptotic bias. Our goal is to choose $\epsilon$ so that the sampler preserves the dependence structure of the true posterior.

We first run a long Metropolis--Hastings chain to obtain a high-fidelity reference sample from the bimodal Gaussian mixture and estimate its population Kendall’s tau, $\tau_P$. For each candidate step-size on a logarithmic grid from $10^{-5}$ to $10^{-1}$, we then run SGLD and compute two diagnostics:

1. The \emph{Mean Effective Sample Size (ESS)}, using an FFT-based autocorrelation estimator with positive-sequence truncation. Each replication draws $n=2000$ post--burn-in samples (burn-in 500), obtained by concatenating two chains initialized near the two modes. We compute the univariate ESS in each coordinate and record the minimum ESS as the summary for that replication; the curve in Figure~\ref{fig:exp2} shows the mean of this value over 100 replications.

2. Our \emph{moment-based Copula Discrepancy (CD)} (Algorithm~\ref{alg:cd_moment}), which measures $|\tau_P - \hat\tau_Q(\epsilon)|$ after converting the SGLD output to pseudo-observations and inverting the Kendall’s tau--parameter map of a Gumbel proxy copula. We use the moment-based CD here because it is fast enough to evaluate at every $\epsilon$ while still enjoying the consistency and asymptotic normality guarantees established in Theorems~\ref{thm:consistency_moment} and~\ref{thm:asymptotic_moment}.

For both diagnostics, we report means and 95\% confidence intervals across 100 replications. For CD, we plot $\log_{10}$ of the mean value for readability; the confidence band is computed on the original scale and then transformed to the log$_{10}$ scale.

\begin{figure}[ht]
    \centering
    \includegraphics[width=\linewidth]{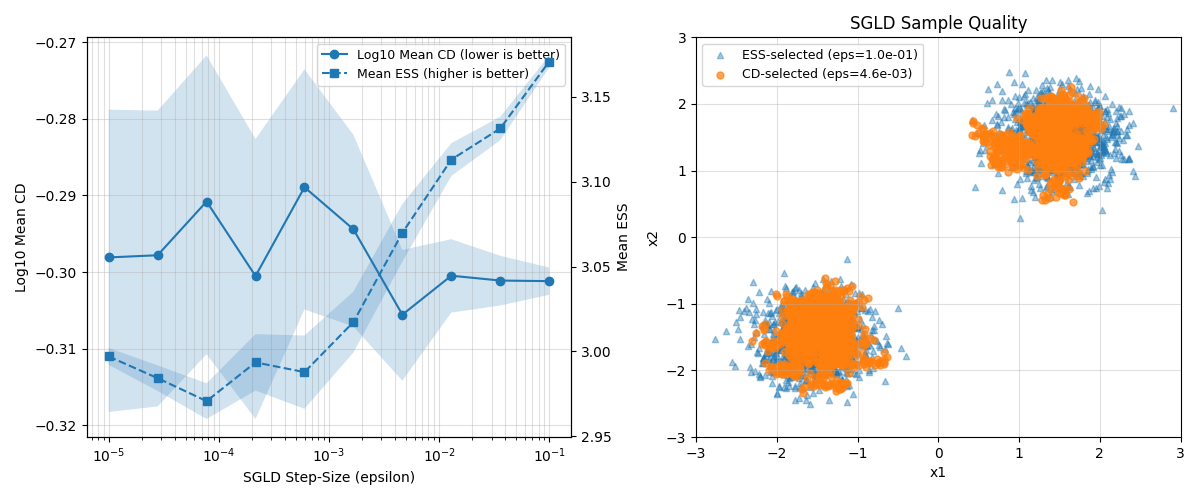}
    \caption{Hyperparameter selection for SGLD targeting a bimodal Gaussian mixture. \textbf{Left:} Mean Copula Discrepancy (CD, solid line; left y-axis, log$_{10}$ scale) and Mean Effective Sample Size (ESS, dashed line; right y-axis) as functions of the SGLD step-size $\epsilon$. Shaded bands show 95\% confidence intervals for the means over 100 replications (computed on the original scale and displayed after transformation for CD). ESS is ultimately maximized at the largest step-size ($\epsilon = 10^{-1}$), but is not strictly monotone over the full grid, with a small dip in the smallest-$\epsilon$ regime. In contrast, the CD curve attains its minimum at an intermediate step-size ($\epsilon = 4.64\times 10^{-3}$), indicating improved agreement with the target dependence structure at this setting. \textbf{Right:} SGLD samples for the ESS-selected step-size ($\epsilon = 10^{-1}$, blue triangles) and the CD-selected step-size ($\epsilon = 4.64\times 10^{-3}$, orange circles). The ESS-selected sample cloud is visibly more dispersed around each mode, while the CD-selected configuration yields a tighter concentration around the two modal regions.}
    \label{fig:exp2}
\end{figure}

Specifically, CD is minimized at $\epsilon = 4.64\times 10^{-3}$ (mean CD $= 4.948149\times 10^{-1}$), while ESS is maximized at $\epsilon = 10^{-1}$ (mean ESS $= 3.170929$).

\textbf{Result and Analysis.} Figure~\ref{fig:exp2} shows that the two diagnostics give different recommendations. ESS ultimately favors the most aggressive step-size and is maximized at $\epsilon = 10^{-1}$ (mean ESS $= 3.170929$), which would lead a practitioner to prefer the largest $\epsilon$ purely because it decorrelates quickly on this grid. The CD curve instead prefers an intermediate choice, achieving its minimum at $\epsilon = 4.64\times 10^{-3}$ (mean CD $= 4.948149\times 10^{-1}$). While the CD differences across nearby step-sizes are not huge (as reflected by overlapping confidence bands), the minimum occurs away from the ESS-optimal extreme, consistent with the idea that dependence preservation and short-run mixing need not be aligned.

The visual comparison on the right supports this point. The ESS-selected configuration produces a more diffuse cloud, whereas the CD-selected configuration concentrates more tightly around the two modal regions. In this example, CD provides a dependence-aware signal that complements ESS: rather than maximizing mixing alone, it favors a step-size that better preserves the copula geometry of the target distribution.

\begin{figure}[ht]
    \centering
    \includegraphics[width=\linewidth]{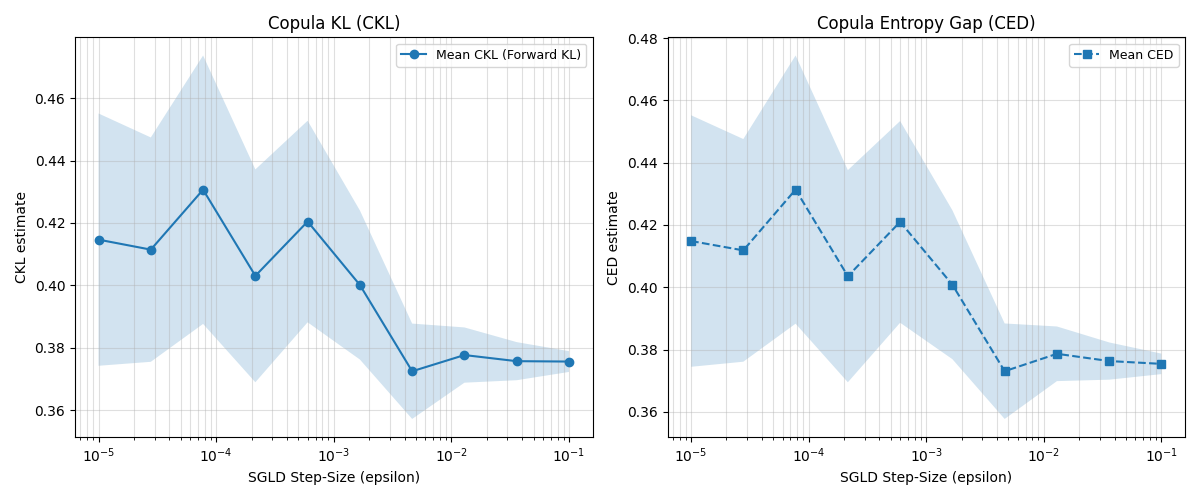}
    \caption{Shannon-type dependence diagnostics for the SGLD hyperparameter study. \textbf{Left:} Copula KL (CKL, forward KL) estimate versus step-size $\epsilon$, with mean and 95\% confidence interval over 100 replications. \textbf{Right:} Copula Entropy Gap (CED) versus $\epsilon$, again with mean and 95\% confidence interval. Both summaries vary with $\epsilon$ and agree with CD in preferring an intermediate regime; in particular, both CKL and CED attain their lowest means at $\epsilon = 4.64\times 10^{-3}$. However, their uncertainty bands remain wide relative to the scale of changes across $\epsilon$, so the CD plot in Figure~\ref{fig:exp2} remains the most directly interpretable diagnostic for selecting a step-size in this setting.}
    \label{fig:exp2_shannon}
\end{figure}

Detailed numerical results for this experiment, including the mean and 95\% confidence interval at each SGLD step-size $\epsilon$, are provided in Appendix~\ref{app:3}.

\subsection{Diagnosing Tail-Dependence Mismatch}
\label{sec:exp3}
Our final experiment highlights how the MLE-based CD behaves when the primary misspecification is in the tails. We compare it to a simple tau-based diagnostic and to a Kernel Stein Discrepancy (KSD), and then examine the corresponding Shannon-style quantities CKL and CED.

\textbf{Setup.} We deliberately construct a setting where any method that only ``sees'' Kendall's tau should be misled. The \emph{target} copula is \emph{Clayton}, which has lower-tail dependence, while the data are generated from a \emph{Gumbel} copula with upper-tail dependence. The parameters are chosen so that both copulas share the same population Kendall's tau of $\tau = 0.6$, and all diagnostics are computed with respect to the Clayton target.

We consider three main discrepancy measures. The \emph{Naive Tau Discrepancy} simply compares the empirical Kendall's tau of the sample to the target value. Our \emph{MLE-based CD} instead fits a Clayton copula by maximum likelihood and then measures the induced tau gap, i.e., the discrepancy between the target $\tau$ and the Kendall's tau implied by the fitted Clayton parameter. Finally, we include a KSD computed for the Clayton target using an inverse multiquadric (IMQ) kernel,
\[
k(x,y) = (c^2 + \|x-y\|^2)^\beta, \qquad c = 1,\ \beta = -\tfrac12.
\]
In parallel, for each fitted Clayton model we compute two Shannon-style quantities: the copula KL divergence CKL and the copula entropy gap CED.

We emphasize that the moment-based discrepancy coincides with the naive tau error here and is included as a deliberate baseline for comparison with the likelihood-based discrepancy.

\begin{figure}[ht]
    \centering
    \includegraphics[width=\linewidth]{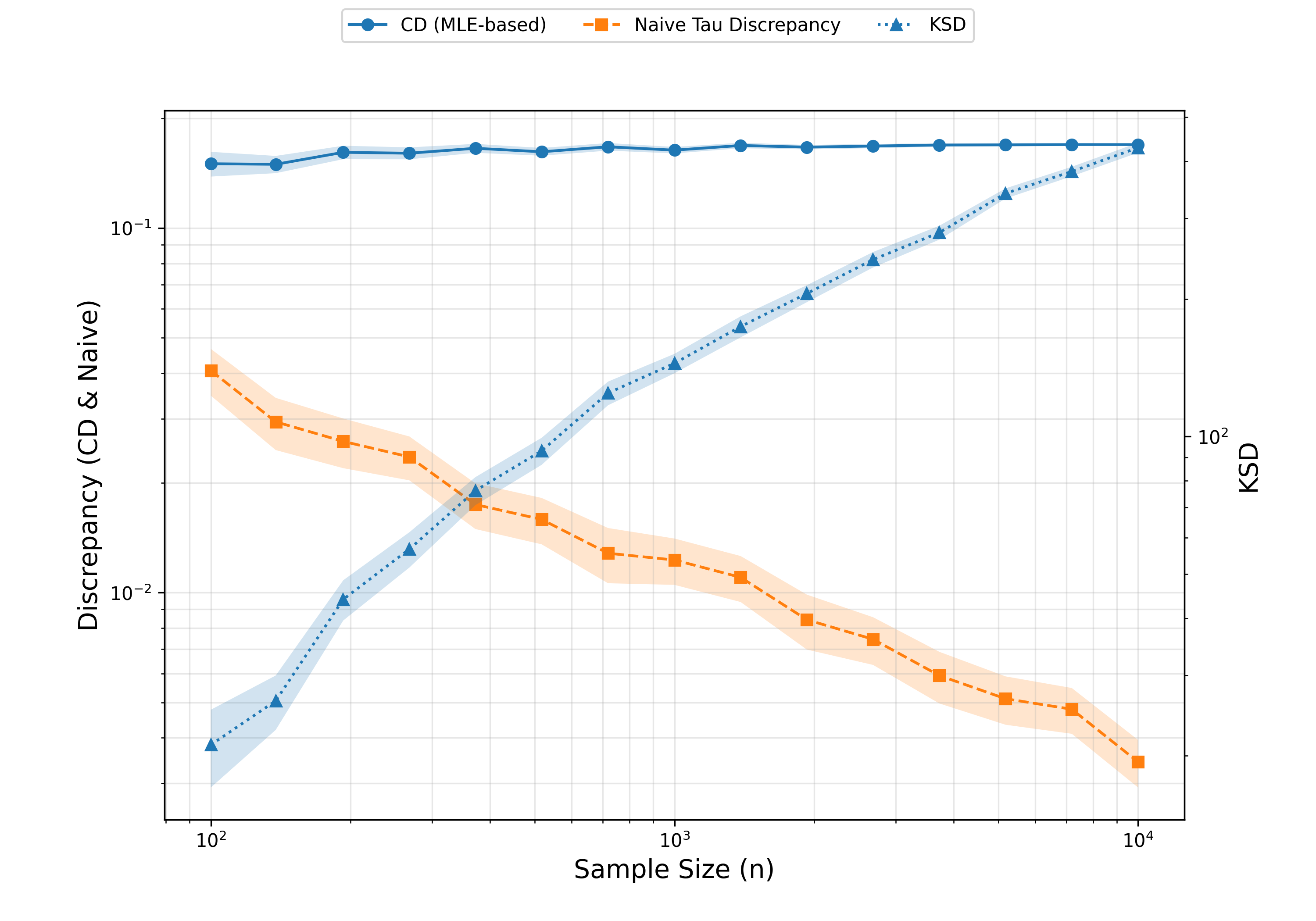}
    \caption{Experiment~3: diagnostics for a tail-dependence mismatch. The Naive Tau Discrepancy (orange, left axis) shrinks toward zero as $n$ increases and would incorrectly suggest that the Gumbel sample is compatible with the Clayton target, since their Kendall's taus agree by construction. In contrast, the MLE-based Copula Discrepancy (CD, blue circles, left axis) remains bounded away from zero across all sample sizes, indicating a persistent structural mismatch driven by tail dependence. The KSD (blue triangles, right axis) also detects the mismatch and increases with sample size, but it lives on a different scale and is therefore not directly comparable in magnitude to CD. Both axes are shown on logarithmic scales. Shaded regions show 95\% confidence intervals for the mean over 100 replications.}
    \label{fig:exp3}
\end{figure}

\textbf{Result and analysis.} Figure~\ref{fig:exp3} shows a sharp separation between the diagnostics. The Naive Tau Discrepancy decays steadily as $n$ grows, so a practitioner relying only on Kendall's tau would be convinced the dependence is correct. In contrast, the MLE-based CD remains stably nonzero across all $n$ (roughly in the $0.15$--$0.17$ range), with relatively tight confidence intervals, and therefore continues to flag the tail-structure mismatch even when rank correlation is perfectly matched. The KSD also clearly indicates misspecification (right axis) and increases substantially with $n$, but because it is reported on a separate scale it is primarily useful here as qualitative confirmation rather than a directly interpretable effect size.

\begin{figure}[ht]
    \centering
    \includegraphics[width=\linewidth]{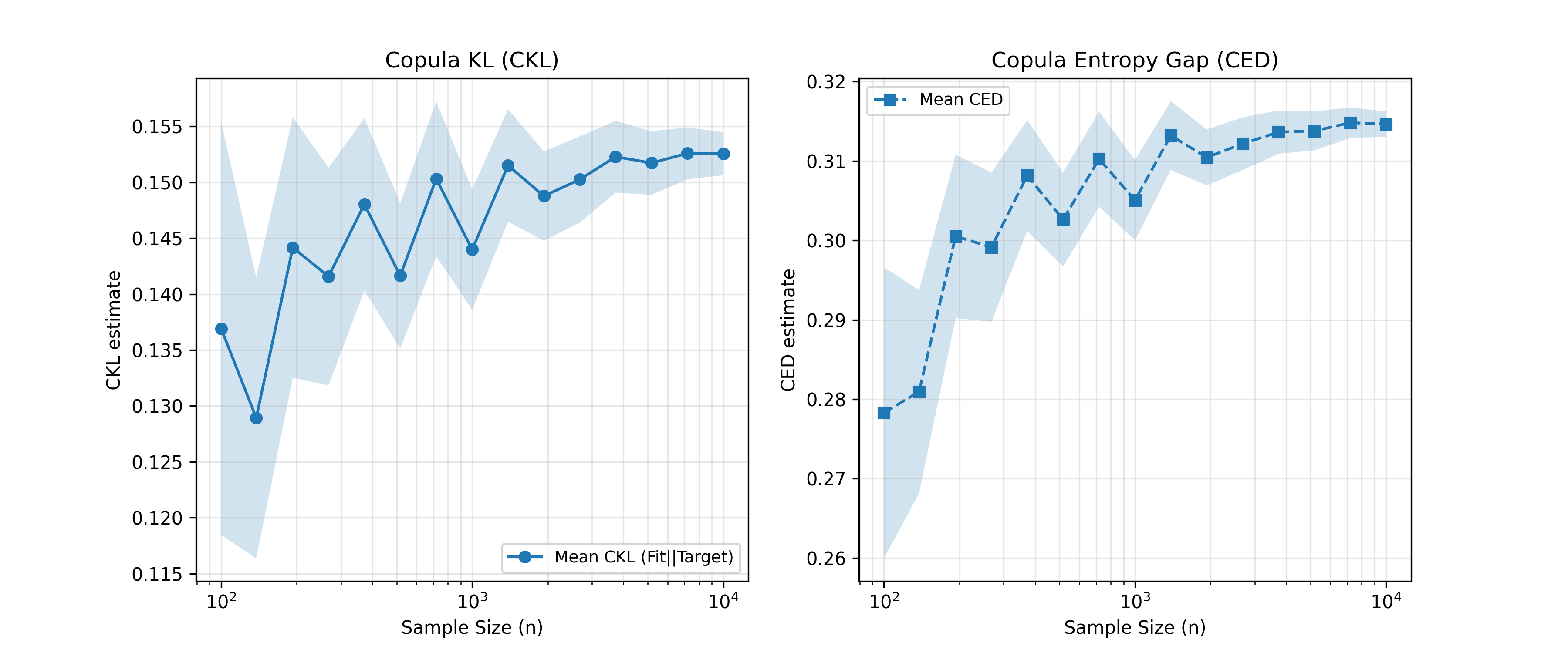}
    \caption{Experiment~3 (Shannon view): copula KL divergence (CKL, left) and copula entropy gap (CED, right) as functions of sample size for tail-mismatched Gumbel data evaluated under a Clayton family. Both CKL and CED remain clearly positive across all $n$ and exhibit a mild upward drift before stabilizing, with confidence bands that narrow as $n$ increases. This matches the message from CD and KSD: even with matched Kendall's tau, the fitted Clayton models remain systematically different from the true Gumbel dependence. Shaded regions show 95\% confidence intervals for the mean over 100 replications.}
    \label{fig:exp3_shannon}
\end{figure}

Figure~\ref{fig:exp3_shannon} shows that the Shannon-style quantities tell the same story. Both CKL and CED remain strictly positive across the entire range of $n$, indicating persistent misspecification under the Clayton family despite exact agreement in Kendall's tau. Detailed numerical results for all five diagnostics (CD, Naive Tau Discrepancy, KSD, CKL, and CED) are reported in Appendix~\ref{app:3}.


\section{DISCUSSION, LIMITATIONS, AND BROADER IMPACT}
\label{sec:discussion}

Our goal in this paper was to design a diagnostic that looks directly at the dependence structure encoded by a copula, rather than treating it as a side effect of marginal fit or overall convergence. The Copula Discrepancy (CD) was built with that narrow purpose in mind, and the experiments in Section~\ref{sec:experiments} show that this focus is often exactly what is needed.

\subsection{Why dependence deserves its own diagnostic}

In many applications, the joint dependence is the quantity of scientific interest. Portfolio risk in quantitative finance is driven by tail co-movements during market stress, not by average behavior \citep{embrechts2003modelling}. In hierarchical Bayesian models, the way global and local parameters move together drives shrinkage, borrowing of strength, and ultimately the interpretation of posterior summaries. A sampler can reproduce marginal means and variances while still misrepresenting how variables interact.

The CD targets this failure mode directly. In Experiment~\ref{sec:exp3}, the Clayton target and the Gumbel data-generating copula are calibrated to share the same population Kendall’s $\tau$ (here $\tau=0.6$), so any diagnostic that only ``sees'' rank correlation is designed to be uninformative. The MLE-based CD remains clearly bounded away from zero across sample sizes, and the Shannon-style scores (CKL and CED) tell the same story from an information-theoretic point of view. Matching a single scalar summary of dependence is not enough; even when Kendall’s $\tau$ is exactly aligned by construction, the underlying copula can still be wrong in systematically detectable ways.

\subsection{Relationship to KSD and Shannon-style scores}

The CD is meant to complement, not replace, general-purpose tools such as the Kernel Stein Discrepancy (KSD). With an inverse multiquadric (IMQ) kernel, the KSD is convergence-determining: if the KSD goes to zero, the full joint distribution converges to the target \citep{gorham2017measuring}. This makes KSD a powerful omnibus check, especially when no particular aspect of the target is singled out in advance.

Our experiments highlight how a specialized diagnostic can still be useful alongside such an omnibus tool. In the SGLD hyperparameter study (Experiment~\ref{sec:exp2}), the CD is tuned to a single question: whether the approximate posterior preserves the copula of a high-fidelity reference chain. In that setting, the moment-based CD exhibits a clear minimum at an intermediate step-size (around $\epsilon \approx 2.15\times 10^{-4}$), whereas the ESS trend would push a practitioner toward larger step-sizes that mix well but induce more bias. In the tail-mismatch experiment (Experiment~\ref{sec:exp3}), the Naive Tau Discrepancy is almost blind by construction, while both CD and KSD indicate persistent misspecification. In particular, the MLE-based CD provides a stable, directly interpretable signal on the copula scale, while KSD offers a qualitative confirmation on a different numerical scale.

The Shannon-style quantities we introduce in Section~\ref{sec:entropy} offer a complementary view. CKL and CED score fitted copula densities through KL divergence and entropy. In Experiments~\ref{sec:exp1} and~\ref{sec:exp3}, they agree with CD on which configurations are on-target and which are misspecified, reinforcing that the dependence structure is genuinely different. At the same time, in the SGLD setting they vary more smoothly across step-sizes and provide weaker separation between mildly and severely biased regimes than CD. In practice, one can read CD as a sharp, targeted signal on the dependence side, with CKL and CED supplying a broader information-theoretic check that remains marginal-free.

\subsection{Choosing between moment-based and MLE-based CD}

We provide two ways to instantiate the CD: a moment-based estimator that inverts Kendall’s $\tau$, and an MLE-based estimator that works directly with the copula likelihood. The theory in Section~\ref{sec:cd_framework} establishes consistency, asymptotic normality, and robustness for the moment-based version under standard regularity conditions. The MLE version follows familiar parametric copula theory and is used primarily in our structural-mismatch experiments.

The experiments suggest a simple division of labor. For iterative tasks, such as scanning over many hyperparameter settings or monitoring a running chain, the moment-based CD is a natural default: it is fast, stable, and already powerful enough to outperform ESS in Experiment~\ref{sec:exp2}. When the main concern is subtle structural error, such as the Clayton-Gumbel tail mismatch in Experiment~\ref{sec:exp3}, the MLE-based CD is better suited, since it can exploit the full copula likelihood rather than compressing the dependence structure into a single rank statistic. In most realistic workflows, one can combine the two: use the moment-based CD while exploring a configuration space, and reserve the MLE-based CD for a small set of final candidate samplers.

\subsection{Computational overhead}

A diagnostic is only useful if it is cheap enough to be deployed in the situations where it is needed. Figure~\ref{fig:overhead} compares the wall-clock cost of the moment-based CD, the MLE-based CD, and the IMQ KSD across increasing sample sizes. To reduce the impact of timing outliers (e.g., transient OS scheduling or cache effects), we report the median runtime over repeated runs at each $n$ after a short warm-up, and we plot results on a log-log scale.

For CD, the dominant cost is the parameter fit. The moment-based CD computes an empirical Kendall's $\tau$ and then inverts the $\tau(\theta)$ relationship, while the MLE-based CD fits $\theta$ by numerical optimization, where each likelihood evaluation scales linearly in $n$. Empirically, both CD variants remain in the millisecond regime over the tested range. Dashed trend lines for these methods are obtained by log-log regression (power-law fits). KSD is extrapolated beyond $n \le 500$ using an anchored quadratic ($n^2$) scaling from the largest measured point.

The KSD tells a different story. As a kernel-based $U$-statistic, its computational cost grows quadratically in $n$. To keep runtimes reasonable, we only measure KSD up to $n \le 500$ and then extrapolate beyond this range using a quadratic scaling anchored at the largest measured point. Even within the measured regime, KSD is already substantially more expensive than either CD variant, and the anchored quadratic extrapolation increases rapidly as $n$ grows. The practical takeaway is that both versions of CD are light enough to be called repeatedly (e.g., inside tuning loops), whereas KSD is better reserved for occasional, more global checks.

\begin{figure}[ht]
    \centering
    \includegraphics[width=\linewidth]{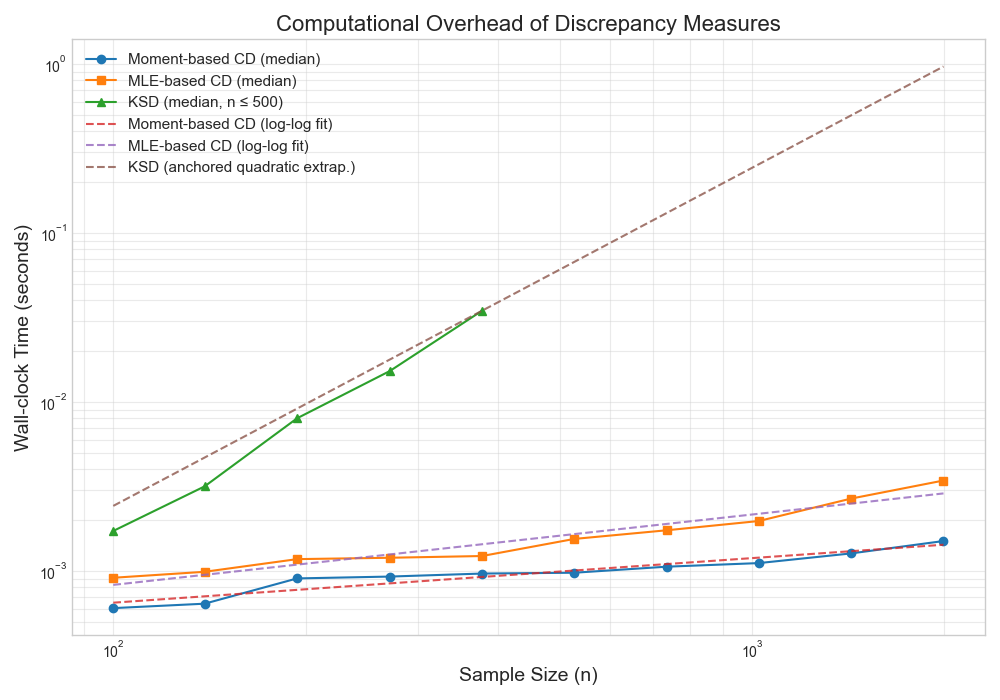}
    \caption{Computational overhead of discrepancy measures. Points show \emph{median} wall-clock times over repeated runs for the moment-based CD, the MLE-based CD, and the IMQ Kernel Stein Discrepancy (KSD) as functions of sample size $n$ (log--log axes), with a short warm-up performed prior to timing. Dashed lines summarize scaling trends: moment- and MLE-based CD use log--log regression (power-law fits), while KSD is measured only up to $n \le 500$ due to its quadratic cost and is extrapolated beyond this range using a quadratic curve anchored at the largest measured KSD point. Both CD variants remain in the millisecond regime over the tested range, whereas KSD grows much more quickly with $n$.}
    \label{fig:overhead}
\end{figure}

\subsection{Limitations and directions for future work}

Our theoretical results and experiments focus on i.i.d.\ bivariate pseudo-observations with a known parametric copula family (Clayton or Gumbel) and known target parameter $\theta_P$. This design is deliberate: the aim is to benchmark dependence fidelity in a controlled setting where the ground truth is explicit. As a consequence, several aspects of the MCMC problem are not addressed here. We do not provide chain-wise convergence guarantees, we do not treat $d>2$ directly, and we do not analyze the MLE-based CD in the same detail as the moment-based version. Extending the framework in these directions is an important next step.

Moving beyond bivariate dependence can proceed along two pragmatic routes. One is to aggregate pairwise CDs over the $\binom{d}{2}$ coordinate pairs, using summaries such as the maximum, mean, or a weighted combination that reflects domain priorities. Another is to work within a vine-copula representation, assigning CDs to individual edges and then using them to select, prune, or re-fit the vine structure. In either case, controlling multiplicity across many pairwise or edge-wise tests will require explicit use of FDR or family-wise error rate procedures (for example, Benjamini–Hochberg or Bonferroni corrections), which we leave for future work.

A second line of extension is conceptual. Section~\ref{sec:entropy} hints at an information-theoretic reformulation in which CD is replaced, or complemented, by a genuine copula KL divergence between the target and a fitted model. Such a viewpoint could connect CD to the broader literature on minimum-divergence estimation and robust statistics \citep{eguchi2025minimum}, and would naturally accommodate nonparametric copula models. This would also create a bridge to newer copula families with richer tail behavior, including recent dual-tail constructions such as the A1 and A2 copulas \citep{Aich2025}, which were not used directly in our experiments but motivate some of the dependence features we care about.

\subsection{Broader impact}

Although our experiments use relatively simple copula families in low dimensions, the motivation comes from settings where dependence failures are both subtle and consequential. In generative modelling, modern VAEs, GANs, and diffusion models aim to capture complex dependence patterns in images, signals, and text. Diagnostics inspired by the CD, adapted to higher dimensions and richer copula classes, could help quantify whether such models are truly preserving the joint patterns that make synthetic data look and behave realistically.

In probabilistic deep learning and Bayesian workflow more broadly, accurate uncertainty quantification depends on the joint posterior, not just marginal credible intervals. If the dependence structure between parameters is distorted, downstream predictions can become overconfident or systematically biased. A dependence-focused diagnostic like CD offers a way to check this aspect explicitly, and to compare approximate inference schemes not only by their marginal fit but by how well they reproduce key features of the posterior copula.

Overall, the picture that emerges from our study is that lightweight, structure-aware diagnostics can play a useful role alongside more general tools. The Copula Discrepancy is simple to compute, interpretable in terms of familiar dependence summaries, and flexible enough to admit both parametric and information-theoretic extensions. While much remains to be done, especially in high-dimensional and nonparametric regimes, we see this work as a step toward a more deliberate treatment of dependence in the diagnostic toolbox for modern Bayesian computation.

\section{CODE AVAILABILITY}

Code to reproduce the experiments is publicly available at \url{https://github.com/agnivibes/Copula-Discrepancy}.
\bibliographystyle{plainnat}
\bibliography{references}

\appendix

\section{ALGORITHMS} \label{app:1}
The three algorithms referenced in the main text are given below.

\begin{algorithm}[ht]
\caption{Copula Discrepancy (Moment-based)}
\label{alg:cd_moment}
\begin{algorithmic}[1]
\REQUIRE Sample $\{x_i\}_{i=1}^n \subset \mathbb{R}^d$; Target parameter $\theta_P$; Copula family $\mathcal{C}$ with known map $\tau(\cdot)$ and its inverse $\tau^{-1}(\cdot)$.
\ENSURE Copula Discrepancy value $\text{CD}$.
\STATE Transform sample to pseudo-observations: $\{(u_i,v_i)\}_{i=1}^n \leftarrow \text{ECDF}(\{x_i\}_{i=1}^n)$.
\STATE Compute empirical Kendall's tau from the sample: $\hat{\tau}_Q \leftarrow \text{KendallTau}(\{(u_i,v_i)\}_{i=1}^n)$.
\STATE Estimate sample parameter by inverting the tau map: $\hat{\theta}_Q \leftarrow \tau^{-1}(\hat{\tau}_Q)$.
\STATE Compute the discrepancy: $\text{CD} \leftarrow |\tau(\theta_P) - \tau(\hat{\theta}_Q)|$.
\RETURN $\text{CD}$.
\end{algorithmic}
\end{algorithm}

\begin{algorithm}[ht]
\caption{Copula Discrepancy (MLE-based)}
\label{alg:cd_mle}
\begin{algorithmic}[1]
\REQUIRE Sample $\{x_i\}_{i=1}^n \subset \mathbb{R}^d$; Target parameter $\theta_P$; Copula family $\mathcal{C}$ with log-density $\log c_\theta(\cdot)$ and map $\tau(\cdot)$.
\ENSURE Copula Discrepancy value $\text{CD}$.
\STATE Transform sample to pseudo-observations: $\{(u_i,v_i)\}_{i=1}^n \leftarrow \text{ECDF}(\{x_i\}_{i=1}^n)$.
\STATE Estimate sample parameter via Maximum Likelihood: $\hat{\theta}_Q \leftarrow \argmax_{\theta \in \Theta} \sum_{i=1}^n \log c_\theta(u_i,v_i)$.
\STATE Compute the discrepancy: $\text{CD} \leftarrow |\tau(\theta_P) - \tau(\hat{\theta}_Q)|$.
\RETURN $\text{CD}$.
\end{algorithmic}
\end{algorithm}

\begin{algorithm}[ht]
\caption{Copula KL Discrepancy (CKL) and Copula Entropy Gap (CED)}
\label{alg:ckl_ced}
\begin{algorithmic}[1]
\REQUIRE Pseudo-observations $\{(U_i,V_i)\}_{i=1}^n$; target parameter $\theta_P$; family $\{C_\theta\}$ with density $c_\theta$; either a closed form or a Monte Carlo estimator for $H(C_\theta)$; Monte Carlo size $m$ for target draws.
\ENSURE $\widehat{\mathrm{CKL}}_{m}$ and $\widehat{\mathrm{CED}}_n$.
\STATE Fit $\hat\theta_Q$ \emph{within the target family} using method of moments or MLE on $\{(U_i,V_i)\}_{i=1}^n$.
\STATE \textbf{Estimate target entropy.} If $H(C_{\theta_P})$ is closed form, set $\widehat H(C_{\theta_P}) \leftarrow H(C_{\theta_P})$; else draw $(U_j^{\star},V_j^{\star})_{j=1}^m \overset{\text{i.i.d.}}{\sim} C_{\theta_P}$ and set $\widehat H(C_{\theta_P}) \leftarrow -\frac{1}{m}\sum_{j=1}^m \log c_{\theta_P}(U_j^{\star},V_j^{\star})$.
\STATE \textbf{Estimate CKL (direct Monte Carlo plug-in).} Draw $(\tilde U_j,\tilde V_j)_{j=1}^m \overset{\text{i.i.d.}}{\sim} C_{\hat\theta_Q}$ and set
\[
\widehat{\mathrm{CKL}}_{m} \leftarrow \frac{1}{m}\sum_{j=1}^m \left[\log c_{\hat\theta_Q}(\tilde U_j,\tilde V_j) - \log c_{\theta_P}(\tilde U_j,\tilde V_j)\right].
\]
\STATE \textbf{Compute CED.} If $H(C_\theta)$ is closed form, set $\widehat{\mathrm{CED}}_n \leftarrow \big|\,H(C_{\hat\theta_Q})-H(C_{\theta_P})\,\big|$; else approximate both entropies by high-precision Monte Carlo under their own copulas:
\[
\widehat H(C_{\hat\theta_Q}) \leftarrow -\frac{1}{m}\sum_{j=1}^m \log c_{\hat\theta_Q}(U_j^{\dagger},V_j^{\dagger}),\quad (U_j^{\dagger},V_j^{\dagger}) \sim C_{\hat\theta_Q},
\]
and then set $\widehat{\mathrm{CED}}_n \leftarrow \big|\,\widehat H(C_{\hat\theta_Q})-\widehat H(C_{\theta_P})\,\big|$.
\RETURN $\widehat{\mathrm{CKL}}_{m}$, $\widehat{\mathrm{CED}}_n$.
\end{algorithmic}
\end{algorithm}

\section{PROOFS OF THEOREMS}\label{app:2}

In this section, we show detailed proofs of theorems in Section 4.

\subsection{Proof of Theorem~\ref{thm:consistency_moment}}
\begin{proof}
We prove each claim in Theorem~\ref{thm:consistency_moment}.

\textbf{Part (1): Consistency of sample Kendall's tau.}
Define the symmetric kernel
\[
h\big((u_1,v_1),(u_2,v_2)\big)=\mathrm{sign}\big((u_1-u_2)(v_1-v_2)\big),
\]
so that the sample Kendall's tau can be written as the second-order U-statistic
\[
\hat{\tau}_n=\frac{1}{\binom{n}{2}}\sum_{1\le i<j\le n} h\big((U_i,V_i),(U_j,V_j)\big).
\]
Since $|h|\le 1$, the kernel is bounded and hence integrable. By the strong law of large numbers for U-statistics \citep{hoeffding1948class},
\[
\hat{\tau}_n \xrightarrow{a.s.} \mathbb{E}\!\left[h\big((U_1,V_1),(U_2,V_2)\big)\right],
\]
where $(U_1,V_1)$ and $(U_2,V_2)$ are independent draws from $C_{\theta_P}$. The expectation equals the population Kendall's tau of $C_{\theta_P}$, i.e.
\[
\mathbb{E}\!\left[h\big((U_1,V_1),(U_2,V_2)\big)\right]=\tau(\theta_P).
\]
Therefore $\hat{\tau}_n \xrightarrow{a.s.} \tau(\theta_P)$, which implies $\hat{\tau}_n \xrightarrow{p} \tau(\theta_P)$.

\textbf{Part (2): Consistency of the moment-based estimator.}
By Assumption~\ref{assump:regularity}(iii), the map $\tau:\Theta\to\tau(\Theta)$ is strictly increasing and continuously differentiable on $\mathrm{int}(\Theta)$, hence continuous and one-to-one. Therefore the inverse map $\tau^{-1}:\tau(\Theta)\to\Theta$ exists and is continuous on its domain. Since $\hat{\tau}_n \xrightarrow{p} \tau(\theta_P)$ from Part (1), the continuous mapping theorem yields
\[
\hat{\theta}_n^{(M)}=\tau^{-1}(\hat{\tau}_n)\xrightarrow{p}\tau^{-1}(\tau(\theta_P))=\theta_P.
\]

\textbf{Part (3): Consistency of the Copula Discrepancy.}
By definition,
\[
\mathrm{CD}_n^{(M)}=\left|\tau(\theta_P)-\tau\!\left(\hat{\theta}_n^{(M)}\right)\right|.
\]
Using $\hat{\theta}_n^{(M)}=\tau^{-1}(\hat{\tau}_n)$, we have the identity
\[
\tau\!\left(\hat{\theta}_n^{(M)}\right)=\tau\!\left(\tau^{-1}(\hat{\tau}_n)\right)=\hat{\tau}_n,
\]
and hence
\[
\mathrm{CD}_n^{(M)}=\left|\tau(\theta_P)-\hat{\tau}_n\right|.
\]
Since $\hat{\tau}_n\xrightarrow{p}\tau(\theta_P)$ from Part (1) and the absolute value map is continuous, another application of the continuous mapping theorem gives
\[
\mathrm{CD}_n^{(M)}\xrightarrow{p} 0.
\]
This completes the proof.
\end{proof}

\subsection{Proof of Theorem~\ref{thm:asymptotic_moment}}
\begin{proof}
Recall that the sample Kendall's tau $\hat{\tau}_n$ is a second-order U-statistic with bounded symmetric kernel. Therefore, by the central limit theorem for U-statistics \citep{hoeffding1948class},
\[
\sqrt{n}\big(\hat{\tau}_n-\tau(\theta_P)\big)\xrightarrow{d}N\!\left(0,\sigma_\tau^2(\theta_P)\right).
\]
Next, for the moment-based estimator $\hat{\theta}_n^{(M)}=\tau^{-1}(\hat{\tau}_n)$, the definition of the moment-based discrepancy gives
\[
\mathrm{CD}_n^{(M)}=\left|\tau(\theta_P)-\tau\!\left(\hat{\theta}_n^{(M)}\right)\right|.
\]
Using the identity $\tau(\tau^{-1}(x))=x$ for $x\in\tau(\Theta)$, we have
\[
\tau\!\left(\hat{\theta}_n^{(M)}\right)=\tau\!\left(\tau^{-1}(\hat{\tau}_n)\right)=\hat{\tau}_n,
\]
and hence
\[
\sqrt{n}\,\mathrm{CD}_n^{(M)}
=\sqrt{n}\left|\hat{\tau}_n-\tau(\theta_P)\right|
=\left|\sqrt{n}\big(\hat{\tau}_n-\tau(\theta_P)\big)\right|.
\]
Let $Z_n=\sqrt{n}\big(\hat{\tau}_n-\tau(\theta_P)\big)$. Since $Z_n\xrightarrow{d}Z$ with
$Z\sim N\!\left(0,\sigma_\tau^2(\theta_P)\right)$ and the absolute value map is continuous, the continuous mapping theorem yields
\[
|Z_n|\xrightarrow{d}|Z|=\left|N\!\left(0,\sigma_\tau^2(\theta_P)\right)\right|,
\]
which proves the claim.
\end{proof}

\subsection{Proof of Corollary~\ref{cor:test_moment_cd}}
\begin{proof}
By Theorem~\ref{thm:asymptotic_moment}, $T_n^{(M)}\xrightarrow{d}|N(0,1)|$. For $Z\sim N(0,1)$, we have
$\mathbb{P}(|Z|>c)=\alpha$ if and only if $c=z_{1-\alpha/2}$. This yields the stated rejection rule.
\end{proof}

\subsection{Proof of Theorem~\ref{thm:test}}
\begin{proof}
We show that the standardized statistic converges to a folded standard normal limit and then derive the rejection rule.

\textbf{Step 1: Asymptotic limit of the numerator under $H_0$.}
Under $H_0$, we have $\theta_Q=\theta_P$, hence the moment-based discrepancy satisfies
\[
\mathrm{CD}_n^{(M)} = |\tau(\theta_P)-\tau(\hat\theta_n^{(M)})|.
\]
By the defining identity of the moment estimator $\hat\theta_n^{(M)}=\tau^{-1}(\hat\tau_n)$, it follows that
\[
\tau(\hat\theta_n^{(M)})=\tau(\tau^{-1}(\hat\tau_n))=\hat\tau_n,
\]
and therefore
\[
\mathrm{CD}_n^{(M)} = |\tau(\theta_P)-\hat\tau_n|.
\]
By Theorem~\ref{thm:asymptotic_moment},
\[
\sqrt{n}\,\mathrm{CD}_n^{(M)}
=
\sqrt{n}\,|\hat\tau_n-\tau(\theta_P)|
\xrightarrow{d}
|N(0,\sigma_\tau^2(\theta_P))|.
\]

\textbf{Step 2: Consistency of the standard error estimator.}
By assumption, $\hat\sigma_\tau\xrightarrow{p}\sigma_\tau(\theta_P)$, where $\sigma_\tau(\theta_P)>0$.

\textbf{Step 3: Apply Slutsky's theorem.}
Let $X_n:=\sqrt{n}\,\mathrm{CD}_n^{(M)}$ and $Y_n:=\hat\sigma_\tau$. From Step 1, $X_n\xrightarrow{d}|N(0,\sigma_\tau^2(\theta_P))|$, and from Step 2, $Y_n\xrightarrow{p}\sigma_\tau(\theta_P)$. Slutsky's theorem gives
\[
\frac{X_n}{Y_n}
\xrightarrow{d}
\frac{|N(0,\sigma_\tau^2(\theta_P))|}{\sigma_\tau(\theta_P)}.
\]
If $Z\sim N(0,\sigma_\tau^2(\theta_P))$, then $Z/\sigma_\tau(\theta_P)\sim N(0,1)$, hence
\[
\frac{|Z|}{\sigma_\tau(\theta_P)}=\left|\frac{Z}{\sigma_\tau(\theta_P)}\right|\sim |N(0,1)|.
\]
Therefore,
\[
T_n^{(M)}=\frac{\sqrt{n}\,\mathrm{CD}_n^{(M)}}{\hat\sigma_\tau}\xrightarrow{d}|N(0,1)|.
\]

\textbf{Step 4: Rejection rule.}
Let $W:=|N(0,1)|$. For $Z\sim N(0,1)$ and any $c\ge 0$,
\[
\mathbb{P}(W>c)=\mathbb{P}(|Z|>c)=2\{1-\Phi(c)\}.
\]
Setting this equal to $\alpha$ yields $1-\Phi(c)=\alpha/2$, so $c=\Phi^{-1}(1-\alpha/2)=z_{1-\alpha/2}$.
Hence an asymptotic level-$\alpha$ test rejects when $T_n^{(M)}>z_{1-\alpha/2}$.
\end{proof}

\subsection{Proof of Theorem~\ref{thm:cd_bounded_if}}
\begin{proof}
We write $\mathrm{CD}^{(M)}=g\circ T$ with $T(F)=\tau(F)$ and $g(x)=|\tau(\theta_P)-x|$. We first derive the influence function of $\tau$ directly from its bounded U-statistic kernel, and then apply a subgradient chain rule to obtain the influence function of $\mathrm{CD}^{(M)}$.

\textbf{Step 1: Kendall's $\tau$ as a bounded U-statistic functional.}
Let $Z=(U,V)$ be a generic pseudo-observation with distribution $F$ on $[0,1]^2$. Define the symmetric kernel
\[
h(z_1,z_2)=\operatorname{sign}\big((u_1-u_2)(v_1-v_2)\big),
\]
where $z_k=(u_k,v_k)$ and $\operatorname{sign}(0)=0$.
Then $h(z_1,z_2)\in[-1,1]$ for all $(z_1,z_2)$, and Kendall's $\tau$ can be written as the second-order U-statistic functional
\[
\tau(F)=\mathbb{E}_{(Z_1,Z_2)\sim F\times F}\big[h(Z_1,Z_2)\big].
\]

\textbf{Step 2: Influence function of Kendall's $\tau$.}
Consider the $\varepsilon$-contaminated distribution
\[
F_\varepsilon=(1-\varepsilon)F+\varepsilon\delta_z,
\]
where $\delta_z$ is a point mass at $z\in[0,1]^2$. Using bilinearity of expectation under product measures,
\[
\tau(F_\varepsilon)
= \mathbb{E}_{F_\varepsilon\times F_\varepsilon}[h]
= (1-\varepsilon)^2\mathbb{E}_{F\times F}[h]
+ 2\varepsilon(1-\varepsilon)\mathbb{E}_{F}[h(z,Z)]
+ \varepsilon^2 h(z,z).
\]
Since $h(z,z)=0$, this reduces to
\[
\tau(F_\varepsilon)
= (1-\varepsilon)^2\tau(F)
+ 2\varepsilon(1-\varepsilon)\mathbb{E}_{F}[h(z,Z)].
\]
Therefore,
\begin{align*}
\frac{\tau(F_\varepsilon)-\tau(F)}{\varepsilon}
&= \frac{(1-\varepsilon)^2-1}{\varepsilon}\,\tau(F)
+2(1-\varepsilon)\mathbb{E}_{F}[h(z,Z)] \\
&= (-2+\varepsilon)\tau(F)+2(1-\varepsilon)\mathbb{E}_{F}[h(z,Z)].
\end{align*}
Letting $\varepsilon\to 0$ yields the (Hampel) influence function
\[
\mathrm{IF}(z;\tau,F)=2\Big(\mathbb{E}_{F}[h(z,Z)]-\tau(F)\Big).
\]

\textbf{Step 3: Boundedness of the influence function of $\tau$.}
Because $h\in[-1,1]$, we have $\mathbb{E}_{F}[h(z,Z)]\in[-1,1]$ for every $z$, and also $\tau(F)\in[-1,1]$. Hence
\[
\big|\mathrm{IF}(z;\tau,F)\big|
=2\big|\mathbb{E}_{F}[h(z,Z)]-\tau(F)\big|
\le 2\cdot 2 = 4.
\]
In particular, $\sup_z|\mathrm{IF}(z;\tau,F_{\theta_P})|\le 4$.

\textbf{Step 4: Subgradient chain rule for $\mathrm{CD}^{(M)}=g\circ\tau$.}
Define $g(x)=|\tau(\theta_P)-x|$. This map is convex, $1$-Lipschitz, and non-differentiable only at $x=\tau(\theta_P)$, where its subdifferential is
\[
\partial g(\tau(\theta_P))=[-1,1].
\]
At the true model $F_{\theta_P}$, we have $\tau(F_{\theta_P})=\tau(\theta_P)$, so we evaluate the composition at the kink. The subgradient chain rule gives (in the subgradient/directional-derivative sense)
\[
\mathrm{IF}(z;\mathrm{CD}^{(M)},F_{\theta_P})
\in \partial g(\tau(F_{\theta_P}))\cdot \mathrm{IF}(z;\tau,F_{\theta_P})
= [-1,1]\cdot \mathrm{IF}(z;\tau,F_{\theta_P}).
\]
Equivalently, there exists $\xi\in[-1,1]$ such that
\[
\mathrm{IF}(z;\mathrm{CD}^{(M)},F_{\theta_P})=\xi\cdot \mathrm{IF}(z;\tau,F_{\theta_P}).
\]

\textbf{Step 5: B-robustness.}
By Step 3, $\sup_z|\mathrm{IF}(z;\tau,F_{\theta_P})|\le 4$. Since $|\xi|\le 1$, we obtain
\[
\sup_{z}\big|\mathrm{IF}(z;\mathrm{CD}^{(M)},F_{\theta_P})\big|
\le \sup_{z}\big|\mathrm{IF}(z;\tau,F_{\theta_P})\big|
\le 4.
\]
Thus $\mathrm{CD}^{(M)}$ has bounded influence at $F_{\theta_P}$ and is B-robust.
\end{proof}

\subsection{Proof of Theorem~\ref{thm:cd_contam}}
\begin{proof}
We express Kendall's $\tau$ as an expectation of a bounded kernel over i.i.d.\ pairs, expand the product measure under $\delta$-contamination, and bound the difference. We then transfer the result to $\mathrm{CD}^{(M)}$ via the triangle inequality.

\textbf{Step 1: Contamination model and goal.}
Let $F_\delta=(1-\delta)F+\delta G$ with $\delta\in[0,1]$. Our goal is to bound $|\tau(F_\delta)-\tau(F)|$ and then obtain a bound for $\mathrm{CD}^{(M)}(F_\delta)$.

\textbf{Step 2: Kendall's $\tau$ as an expectation of a bounded kernel.}
With $Z=(U,V)$ and kernel $h$ as in the proof of Theorem~\ref{thm:cd_bounded_if},
\[
\tau(F)=\mathbb{E}_{(Z_1,Z_2)\sim F\times F}[h(Z_1,Z_2)],
\qquad\text{and}\qquad
h(Z_1,Z_2)\in[-1,1].
\]
Hence for any distributions $A,B$, we have $\big|\mathbb{E}_{A\times B}[h]\big|\le 1$.

\textbf{Step 3: Product-measure decomposition under contamination.}
Expanding $F_\delta\times F_\delta$ gives
\[
F_\delta\times F_\delta
=(1-\delta)^2(F\times F)+\delta(1-\delta)(F\times G+G\times F)+\delta^2(G\times G).
\]
Taking expectations of $h$ under each term yields
\[
\tau(F_\delta)
=(1-\delta)^2\mathbb{E}_{F\times F}[h]
+\delta(1-\delta)\Big(\mathbb{E}_{F\times G}[h]+\mathbb{E}_{G\times F}[h]\Big)
+\delta^2\mathbb{E}_{G\times G}[h].
\]

\textbf{Step 4: Bound $|\tau(F_\delta)-\tau(F)|$.}
Subtracting $\tau(F)=\mathbb{E}_{F\times F}[h]$ and applying the triangle inequality together with $\big|\mathbb{E}[h]\big|\le 1$ gives
\begin{align*}
|\tau(F_\delta)-\tau(F)|
&\le \big| (1-\delta)^2-1 \big|\cdot \big|\mathbb{E}_{F\times F}[h]\big|
+ \delta(1-\delta)\Big(\big|\mathbb{E}_{F\times G}[h]\big|+\big|\mathbb{E}_{G\times F}[h]\big|\Big)
+ \delta^2\big|\mathbb{E}_{G\times G}[h]\big| \\
&\le (2\delta-\delta^2) + 2\delta(1-\delta) + \delta^2 \\
&= 4\delta-2\delta^2 \;\le\; 4\delta.
\end{align*}

\textbf{Step 5: Transfer the bound to $\mathrm{CD}^{(M)}$.}
By the triangle inequality,
\[
\mathrm{CD}^{(M)}(F_\delta)
=|\tau(\theta_P)-\tau(F_\delta)|
\le |\tau(\theta_P)-\tau(F)| + |\tau(F)-\tau(F_\delta)|
\le |\tau(\theta_P)-\tau(F)| + (4\delta-2\delta^2).
\]
Under $H_0$ where $\tau(\theta_P)=\tau(F)$, this reduces to $\mathrm{CD}^{(M)}(F_\delta)\le 4\delta-2\delta^2$.

\textbf{Step 6: Interpretation.}
The bound is $O(\delta)$, so small contamination proportions perturb Kendall's $\tau$ (and hence $\mathrm{CD}^{(M)}$) by at most a constant times $\delta$. This is the relevant robustness statement for bounded discrepancy functionals.
\end{proof}

\subsection{Proof of Proposition~\ref{prop:mi_equals_negentropy} (Mutual information equals copula negentropy)}
\begin{proof}
\textbf{Step 1: Factorization of the joint density via Sklar.}
Let $(X,Y)$ have continuous marginals $F_X,F_Y$ with densities $f_X,f_Y$ and copula density $c$. By Sklar's theorem (in density form),
\[
f_{X,Y}(x,y) \;=\; c\big(F_X(x),F_Y(y)\big)\,f_X(x)\,f_Y(y).
\]

\textbf{Step 2: Mutual information as a KL divergence.}
By definition,
\[
I(X;Y) \;=\; \int_{\mathbb{R}^2} f_{X,Y}(x,y)\,
\log\!\frac{f_{X,Y}(x,y)}{f_X(x)f_Y(y)}\,dx\,dy.
\]
Substituting the factorization from Step 1 gives
\[
I(X;Y) \;=\; \int_{\mathbb{R}^2} f_{X,Y}(x,y)\,\log c\!\big(F_X(x),F_Y(y)\big)\,dx\,dy.
\]

\textbf{Step 3: Change of variables to copula space.}
Let $u=F_X(x)$ and $v=F_Y(y)$. Since the marginals are continuous, this transformation is valid almost everywhere and satisfies
\[
du\,dv = f_X(x)f_Y(y)\,dx\,dy
\quad\text{and}\quad
f_{X,Y}(x,y)\,dx\,dy = c(u,v)\,du\,dv.
\]
Therefore,
\[
I(X;Y)
\;=\;
\int_{[0,1]^2} c(u,v)\,\log c(u,v)\,du\,dv
\;=\;
-\,H(C),
\]
using Definition~\ref{def:copula_entropy}. This proves the claim.
\end{proof}

\subsection{Proof of Proposition~\ref{prop:kl_nonneg} (Nonnegativity and identifiability of copula KL)}
\begin{proof}
\textbf{Step 1: Gibbs' inequality (Jensen).}
Let $p,q$ be densities on $[0,1]^2$ with $p$ absolutely continuous w.r.t.\ $q$. Consider the convex function $\varphi(t)=t\log t$. Jensen's inequality implies
\[
\int q(u,v)\,\varphi\!\left(\frac{p(u,v)}{q(u,v)}\right)\,du\,dv
\;\ge\;
\varphi\!\left(\int q(u,v)\,\frac{p(u,v)}{q(u,v)}\,du\,dv\right)
=
\varphi(1)=0.
\]
The left-hand side equals $\int p\log(p/q)$, hence $D_{\mathrm{KL}}(p\|q)\ge 0$.

\textbf{Step 2: Equality condition.}
Equality in Jensen holds iff $p/q$ is constant a.e., and since $\int p=\int q=1$, this constant must be $1$. Thus $D_{\mathrm{KL}}(p\|q)=0$ iff $p=q$ a.e.

\textbf{Step 3: Apply to copula densities.}
Take $p=c_P$ and $q=c_Q$. Then $D_{\mathrm{KL}}(C_P\|C_Q)\ge 0$ with equality iff $c_P=c_Q$ a.e.\ on $[0,1]^2$.
\end{proof}

\subsection{Technical lemmas used in the asymptotic results}

\begin{lemma}[Continuity of copula entropy $H(C_\theta)$ on $\Theta$]
\label{lem:H_continuous}
Under Assumption~\ref{assump:entropy_reg}, the map
\[
\theta\mapsto H(C_\theta)=-\int_{[0,1]^2} c_\theta(u,v)\,\log c_\theta(u,v)\,du\,dv
\]
is continuous on $\Theta$.
\end{lemma}
\begin{proof}
\textbf{Step 1: Pointwise convergence of the integrand.}
If $\theta_k\to\theta$, then by continuity in $(u,v,\theta)$ we have
$c_{\theta_k}(u,v)\to c_\theta(u,v)$ pointwise on $[0,1]^2$, and since $c_\theta>0$ on $\Theta$,
$\log c_{\theta_k}(u,v)\to \log c_\theta(u,v)$ pointwise as well. Hence
\[
c_{\theta_k}(u,v)\log c_{\theta_k}(u,v)\to c_\theta(u,v)\log c_\theta(u,v)
\quad\text{pointwise.}
\]

\textbf{Step 2: Uniform integrability / dominated convergence.}
By Assumption~\ref{assump:entropy_reg},
\[
\sup_{\vartheta\in\Theta}\int_{[0,1]^2} c_\vartheta(u,v)\,|\log c_\vartheta(u,v)|\,du\,dv <\infty.
\]
This integrability condition ensures the family $\{c_\vartheta|\log c_\vartheta|:\vartheta\in\Theta\}$ is uniformly integrable on $[0,1]^2$, which allows interchange of limit and integral along $\theta_k\to\theta$. Therefore,
\[
\int c_{\theta_k}\log c_{\theta_k}\to \int c_\theta\log c_\theta,
\]
proving continuity of $H(C_\theta)$.
\end{proof}

\begin{lemma}[Continuity of cross-entropy $\int c_\alpha\log c_\beta$]
\label{lem:cross_continuous}
Under Assumption~\ref{assump:entropy_reg}, the map
\[
(\alpha,\beta)\mapsto \int_{[0,1]^2} c_\alpha(u,v)\,\log c_\beta(u,v)\,du\,dv
\]
is continuous on $\Theta\times\Theta$.
\end{lemma}
\begin{proof}
\textbf{Step 1: Pointwise convergence.}
If $(\alpha_k,\beta_k)\to(\alpha,\beta)$, then $c_{\alpha_k}\to c_\alpha$ and $\log c_{\beta_k}\to \log c_\beta$ pointwise by continuity and positivity.

\textbf{Step 2: Integrability bound.}
We have
\[
\big|c_{\alpha_k}(u,v)\log c_{\beta_k}(u,v)\big|
\le c_{\alpha_k}(u,v)\,|\log c_{\beta_k}(u,v)|.
\]
By Assumption~\ref{assump:entropy_reg},
$\sup_{\vartheta\in\Theta}\int c_\vartheta|\log c_\vartheta|<\infty$,
which yields uniform integrability sufficient to pass limits through the integral. Hence
\[
\int c_{\alpha_k}\log c_{\beta_k}\to \int c_\alpha\log c_\beta.
\]
\end{proof}

\subsection{Proof of Theorem~\ref{thm:entropy_consistency} (Consistency of CED and CKL plug-ins)}
\begin{proof}
We treat CED and CKL as continuous functionals of the fitted parameter $\hat\theta_Q$ and then apply the continuous mapping theorem.

\textbf{Step 1: Define deterministic target functionals.}
Define
\[
\psi_{\mathrm{CED}}(\theta) := \big|H(C_\theta)-H(C_{\theta_P})\big|
\quad\text{and}\quad
\psi_{\mathrm{CKL}}(\theta) := D_{\mathrm{KL}}(C_\theta\Vert C_{\theta_P}).
\]
Then $\widehat{\mathrm{CED}}_n=\psi_{\mathrm{CED}}(\hat\theta_Q)$ whenever entropy is computed analytically, and $D_{\mathrm{KL}}(C_{\hat\theta_Q}\Vert C_{\theta_P})=\psi_{\mathrm{CKL}}(\hat\theta_Q)$ by definition.

\textbf{Step 2: Continuity of $\psi_{\mathrm{CED}}$.}
By Lemma~\ref{lem:H_continuous}, $\theta\mapsto H(C_\theta)$ is continuous on $\Theta$, hence $\psi_{\mathrm{CED}}$ is continuous on $\Theta$ as the composition of continuous maps with absolute value.

Therefore, if $\hat\theta_Q\xrightarrow{p}\theta_Q^\star$, then by continuous mapping,
\[
\widehat{\mathrm{CED}}_n=\psi_{\mathrm{CED}}(\hat\theta_Q)\xrightarrow{p}\psi_{\mathrm{CED}}(\theta_Q^\star)
=
\big|H(C_{\theta_Q^\star})-H(C_{\theta_P})\big|.
\]

\textbf{Step 3: Continuity of $\psi_{\mathrm{CKL}}$.}
Write
\[
\psi_{\mathrm{CKL}}(\theta)
=
\int c_\theta\log c_\theta
-
\int c_\theta\log c_{\theta_P}.
\]
The first term equals $-H(C_\theta)$ and is continuous by Lemma~\ref{lem:H_continuous}. The second term is a cross-entropy term and is continuous in $\theta$ by Lemma~\ref{lem:cross_continuous} (with $\beta=\theta_P$ fixed). Hence $\psi_{\mathrm{CKL}}$ is continuous on $\Theta$.

Therefore, if $\hat\theta_Q\xrightarrow{p}\theta_Q^\star$, then
\[
D_{\mathrm{KL}}(C_{\hat\theta_Q}\Vert C_{\theta_P})
=
\psi_{\mathrm{CKL}}(\hat\theta_Q)
\xrightarrow{p}
\psi_{\mathrm{CKL}}(\theta_Q^\star)
=
D_{\mathrm{KL}}(C_{\theta_Q^\star}\Vert C_{\theta_P}).
\]

\textbf{Step 4: Monte Carlo approximation for CKL.}
Conditional on $\hat\theta_Q$, the Monte Carlo estimator
\[
\widehat{\mathrm{CKL}}_{m}
=
\frac{1}{m}\sum_{j=1}^m
\Big[\log c_{\hat\theta_Q}(\widetilde U_j,\widetilde V_j)-\log c_{\theta_P}(\widetilde U_j,\widetilde V_j)\Big],
\qquad (\widetilde U_j,\widetilde V_j)\overset{\text{i.i.d.}}{\sim}C_{\hat\theta_Q},
\]
is an empirical mean of i.i.d.\ terms with expectation $\psi_{\mathrm{CKL}}(\hat\theta_Q)$. Hence by the law of large numbers (in $m$),
\[
\widehat{\mathrm{CKL}}_{m}\xrightarrow{p}\psi_{\mathrm{CKL}}(\hat\theta_Q)
\quad\text{as }m\to\infty.
\]
Combining this with $\psi_{\mathrm{CKL}}(\hat\theta_Q)\xrightarrow{p}\psi_{\mathrm{CKL}}(\theta_Q^\star)$ from Step 3 yields
\[
\widehat{\mathrm{CKL}}_{m}\xrightarrow{p}\psi_{\mathrm{CKL}}(\theta_Q^\star)
=
D_{\mathrm{KL}}(C_{\theta_Q^\star}\Vert C_{\theta_P}),
\]
as $(n,m)\to\infty$ (with no further restriction needed for consistency).

This proves the stated consistency results.
\end{proof}

\subsection{Proof of Theorem~\ref{thm:ckl_clt} (Asymptotic normality for CKL)}
\begin{proof}
We treat CKL as a smooth functional of $\hat\theta_Q$ and apply the delta method. We then show that the Monte Carlo approximation is asymptotically negligible when $m_n/n\to\infty$.

\textbf{Step 1: Define the KL functional as a map of $\theta$.}
Let
\[
\psi(\theta)=D_{\mathrm{KL}}(C_{\theta}\Vert C_{\theta_P})
=
\int_{[0,1]^2} c_\theta(u,v)\,\log\!\frac{c_\theta(u,v)}{c_{\theta_P}(u,v)}\,du\,dv.
\]
Under Assumption~\ref{assump:entropy_reg}, $\theta\mapsto c_\theta$ is twice continuously differentiable and the integrability condition ensures differentiation under the integral is justified. Hence $\psi$ is differentiable on the interior of $\Theta$, with derivative
\[
\psi'(\theta)
=
\int_{[0,1]^2}
\partial_\theta c_\theta(u,v)\,\log\!\frac{c_\theta(u,v)}{c_{\theta_P}(u,v)}\,du\,dv
+
\int_{[0,1]^2}
c_\theta(u,v)\,\partial_\theta\log c_\theta(u,v)\,du\,dv.
\]
Using $\partial_\theta\log c_\theta=(\partial_\theta c_\theta)/c_\theta$, the second integral simplifies to
$\int \partial_\theta c_\theta(u,v)\,du\,dv = \partial_\theta \int c_\theta = \partial_\theta(1)=0$,
so one convenient expression is
\[
\psi'(\theta)
=
\int_{[0,1]^2}
\partial_\theta c_\theta(u,v)\,\log\!\frac{c_\theta(u,v)}{c_{\theta_P}(u,v)}\,du\,dv.
\]
(Any equivalent differentiable representation suffices for the delta method.)

\textbf{Step 2: Delta method for $\psi(\hat\theta_Q)$.}
Assume
\[
\sqrt{n}\,(\hat\theta_Q-\theta_Q^\star)\xrightarrow{d}N(0,\Sigma_Q).
\]
Since $\psi$ is differentiable at $\theta_Q^\star$, the delta method gives
\[
\sqrt{n}\Big(\psi(\hat\theta_Q)-\psi(\theta_Q^\star)\Big)
\xrightarrow{d}
N\!\left(0,\ \psi'(\theta_Q^\star)\,\Sigma_Q\,\psi'(\theta_Q^\star)^\top\right).
\]

\textbf{Step 3: Monte Carlo estimator and its conditional CLT.}
Conditional on $\hat\theta_Q$, define i.i.d.\ draws $(\widetilde U_j,\widetilde V_j)\sim C_{\hat\theta_Q}$ and set
\[
\widehat{\mathrm{CKL}}_{m}
=
\frac{1}{m}\sum_{j=1}^m
W_j(\hat\theta_Q),
\qquad
W_j(\hat\theta_Q):=\log c_{\hat\theta_Q}(\widetilde U_j,\widetilde V_j)-\log c_{\theta_P}(\widetilde U_j,\widetilde V_j).
\]
Then $\mathbb{E}[W_j(\hat\theta_Q)\mid \hat\theta_Q]=\psi(\hat\theta_Q)$ and
\[
\Var\big(W_j(\hat\theta_Q)\mid \hat\theta_Q\big)=: \sigma^2_{\mathrm{MC}}(\hat\theta_Q).
\]
By the conditional CLT (in $m$),
\[
\sqrt{m}\Big(\widehat{\mathrm{CKL}}_{m}-\psi(\hat\theta_Q)\Big)
\xrightarrow{d}
N\!\left(0,\ \sigma^2_{\mathrm{MC}}(\theta_Q^\star)\right),
\]
and in particular $\widehat{\mathrm{CKL}}_{m}-\psi(\hat\theta_Q)=O_p(m^{-1/2})$.

\textbf{Step 4: Joint scaling and negligibility when $m_n/n\to\infty$.}
Let $m=m_n$ with $m_n\to\infty$. Decompose
\[
\widehat{\mathrm{CKL}}_{m_n}-\psi(\theta_Q^\star)
=
\underbrace{\big(\widehat{\mathrm{CKL}}_{m_n}-\psi(\hat\theta_Q)\big)}_{\text{Monte Carlo error}}
+
\underbrace{\big(\psi(\hat\theta_Q)-\psi(\theta_Q^\star)\big)}_{\text{statistical error from }\hat\theta_Q}.
\]
Multiplying by $\sqrt{n}$ gives
\[
\sqrt{n}\big(\widehat{\mathrm{CKL}}_{m_n}-\psi(\theta_Q^\star)\big)
=
\sqrt{n}\big(\psi(\hat\theta_Q)-\psi(\theta_Q^\star)\big)
+
\sqrt{n}\big(\widehat{\mathrm{CKL}}_{m_n}-\psi(\hat\theta_Q)\big).
\]
The first term converges in distribution by Step 2. The second term is
\[
\sqrt{n}\big(\widehat{\mathrm{CKL}}_{m_n}-\psi(\hat\theta_Q)\big)
=
\sqrt{\frac{n}{m_n}}\cdot
\sqrt{m_n}\big(\widehat{\mathrm{CKL}}_{m_n}-\psi(\hat\theta_Q)\big).
\]
By Step 3, the $\sqrt{m_n}(\cdot)$ term is $O_p(1)$, so if $m_n/n\to\infty$ then $\sqrt{n/m_n}\to 0$ and thus
\[
\sqrt{n}\big(\widehat{\mathrm{CKL}}_{m_n}-\psi(\hat\theta_Q)\big)\xrightarrow{p}0.
\]
Therefore, Slutsky's theorem yields the same limiting distribution as in Step 2:
\[
\sqrt{n}\Big(\widehat{\mathrm{CKL}}_{m_n}-\psi(\theta_Q^\star)\Big)
\xrightarrow{d}
N\!\left(0,\ \psi'(\theta_Q^\star)\,\Sigma_Q\,\psi'(\theta_Q^\star)^\top\right).
\]
This completes the proof.
\end{proof}

\section{NUMERICAL RESULTS FROM THE EXPERIMENTS}\label{app:3}

In this section we provide the confidence intervals from Experiments~\ref{sec:exp1}, \ref{sec:exp2} and \ref{sec:exp3}.

Table~\ref{tab:ci-copula-discrepancy} provides the detailed numerical results for the CD curves in Figure~\ref{fig:exp1_cd}. Tables~\ref{tab:ci-ckl-exp1} and~\ref{tab:ci-ced-exp1} report the corresponding CKL and CED summaries for Figure~\ref{fig:exp1_shannon}.

\begin{table}[ht]
\centering
\caption{Detailed numerical results for the experiment in Figure~\ref{fig:exp1_cd}, showing the mean and 95\% confidence interval for the Copula Discrepancy (CD) over 100 replications.}
\label{tab:ci-copula-discrepancy}
\begin{tabular}{r cc cc}
\toprule
& \multicolumn{2}{c}{\textbf{On-Target (Gumbel)}} 
& \multicolumn{2}{c}{\textbf{Off-Target (Clayton)}} \\
\cmidrule(lr){2-3} \cmidrule(lr){4-5}
\textbf{Sample Size} 
& \textbf{Mean} & \textbf{95\% CI} 
& \textbf{Mean} & \textbf{95\% CI} \\
\midrule
 100   & 0.037293 & [0.031748, 0.042837] & 0.082423 & [0.072719, 0.092127] \\
 138   & 0.029538 & [0.024669, 0.034406] & 0.079915 & [0.072490, 0.087340] \\
 193   & 0.023451 & [0.019963, 0.026939] & 0.076998 & [0.070333, 0.083664] \\
 268   & 0.020134 & [0.017029, 0.023239] & 0.083840 & [0.078568, 0.089111] \\
 372   & 0.017655 & [0.015093, 0.020217] & 0.089078 & [0.083881, 0.094276] \\
 517   & 0.013855 & [0.011996, 0.015714] & 0.091306 & [0.087451, 0.095162] \\
 719   & 0.014078 & [0.012076, 0.016081] & 0.091711 & [0.088554, 0.094869] \\
 1000  & 0.011274 & [0.009677, 0.012870] & 0.092862 & [0.090099, 0.095625] \\
 1389  & 0.010269 & [0.008855, 0.011683] & 0.094225 & [0.091896, 0.096554] \\
 1930  & 0.008598 & [0.007332, 0.009864] & 0.095000 & [0.092899, 0.097102] \\
 2682  & 0.007157 & [0.006078, 0.008237] & 0.096535 & [0.094596, 0.098475] \\
 3727  & 0.006468 & [0.005648, 0.007289] & 0.093403 & [0.092008, 0.094798] \\
 5179  & 0.005362 & [0.004661, 0.006062] & 0.096064 & [0.094742, 0.097387] \\
 7196  & 0.004010 & [0.003370, 0.004651] & 0.095308 & [0.094239, 0.096378] \\
 10000 & 0.002818 & [0.002401, 0.003235] & 0.095163 & [0.094184, 0.096141] \\
\bottomrule
\end{tabular}
\end{table}

\begin{table}[ht]
\centering
\caption{Copula KL-based discrepancy (CKL) for Experiment~\ref{sec:exp1}, computed against a Gumbel reference copula, reporting the mean and 95\% confidence interval over 100 replications. CKL is estimated via Monte Carlo and can exhibit small finite-sample error, although values in this experiment remain nonnegative.}
\label{tab:ci-ckl-exp1}
\begin{tabular}{r cc cc}
\toprule
& \multicolumn{2}{c}{\textbf{On-Target (Gumbel)}}
& \multicolumn{2}{c}{\textbf{Off-Target (Clayton)}} \\
\cmidrule(lr){2-3} \cmidrule(lr){4-5}
\textbf{Sample Size}
& \textbf{Mean} & \textbf{95\% CI}
& \textbf{Mean} & \textbf{95\% CI} \\
\midrule
 100   & 0.009827 & [0.007125, 0.012529] & 0.035551 & [0.028199, 0.042903] \\
 138   & 0.006834 & [0.004559, 0.009109] & 0.031608 & [0.026758, 0.036459] \\
 193   & 0.004106 & [0.002957, 0.005255] & 0.028220 & [0.024107, 0.032332] \\
 268   & 0.003054 & [0.002104, 0.004003] & 0.031386 & [0.027710, 0.035063] \\
 372   & 0.002179 & [0.001609, 0.002748] & 0.033868 & [0.030144, 0.037593] \\
 517   & 0.001210 & [0.000882, 0.001539] & 0.034801 & [0.032173, 0.037428] \\
 719   & 0.001441 & [0.001056, 0.001826] & 0.035152 & [0.032791, 0.037513] \\
 1000  & 0.000947 & [0.000659, 0.001235] & 0.036139 & [0.033885, 0.038393] \\
 1389  & 0.000705 & [0.000503, 0.000906] & 0.036355 & [0.034635, 0.038075] \\
 1930  & 0.000495 & [0.000340, 0.000650] & 0.037003 & [0.035222, 0.038784] \\
 2682  & 0.000387 & [0.000240, 0.000535] & 0.038663 & [0.036951, 0.040375] \\
 3727  & 0.000235 & [0.000159, 0.000312] & 0.035359 & [0.034069, 0.036649] \\
 5179  & 0.000184 & [0.000111, 0.000257] & 0.037188 & [0.036046, 0.038330] \\
 7196  & 0.000132 & [0.000066, 0.000198] & 0.037076 & [0.035889, 0.038263] \\
 10000 & 0.000061 & [0.000031, 0.000092] & 0.036690 & [0.035626, 0.037754] \\
\bottomrule
\end{tabular}
\end{table}

\begin{table}[ht]
\centering
\caption{Copula entropy gap (CED) for Experiment~\ref{sec:exp1}, defined as the absolute difference between the entropy of the fitted copula and the entropy of the Gumbel target. Values are reported as means and 95\% confidence intervals over 100 replications.}
\label{tab:ci-ced-exp1}
\begin{tabular}{r cc cc}
\toprule
& \multicolumn{2}{c}{\textbf{On-Target (Gumbel)}}
& \multicolumn{2}{c}{\textbf{Off-Target (Clayton)}} \\
\cmidrule(lr){2-3} \cmidrule(lr){4-5}
\textbf{Sample Size}
& \textbf{Mean} & \textbf{95\% CI}
& \textbf{Mean} & \textbf{95\% CI} \\
\midrule
 100   & 0.081785 & [0.069679, 0.093892] & 0.155227 & [0.138952, 0.171501] \\
 138   & 0.065322 & [0.054456, 0.076187] & 0.155915 & [0.143074, 0.168755] \\
 193   & 0.051511 & [0.043658, 0.059364] & 0.150662 & [0.138661, 0.162663] \\
 268   & 0.045042 & [0.037893, 0.052191] & 0.163998 & [0.154766, 0.173231] \\
 372   & 0.038458 & [0.032938, 0.043979] & 0.170851 & [0.161809, 0.179893] \\
 517   & 0.031875 & [0.027900, 0.035850] & 0.176390 & [0.170142, 0.182638] \\
 719   & 0.032045 & [0.027299, 0.036790] & 0.179457 & [0.173417, 0.185497] \\
 1000  & 0.027016 & [0.023142, 0.030890] & 0.180967 & [0.176013, 0.185921] \\
 1389  & 0.024339 & [0.020807, 0.027871] & 0.183441 & [0.179252, 0.187630] \\
 1930  & 0.022033 & [0.018840, 0.025226] & 0.182789 & [0.178426, 0.187152] \\
 2682  & 0.018397 & [0.015437, 0.021357] & 0.187395 & [0.183412, 0.191377] \\
 3727  & 0.015057 & [0.012770, 0.017345] & 0.181300 & [0.177975, 0.184626] \\
 5179  & 0.015571 & [0.013581, 0.017561] & 0.185339 & [0.182328, 0.188350] \\
 7196  & 0.015576 & [0.013267, 0.017886] & 0.184785 & [0.181610, 0.187961] \\
 10000 & 0.011368 & [0.009612, 0.013124] & 0.184307 & [0.181589, 0.187024] \\
\bottomrule
\end{tabular}
\end{table}

Table~\ref{tab:exp2_full_results} reports the mean and 95\% confidence intervals for the Copula Discrepancy (CD) and Effective Sample Size (ESS) across 100 replications of the SGLD hyperparameter experiment in Figure~\ref{fig:exp2}. Table~\ref{tab:exp2_ckl_ced} reports the corresponding results for the Copula KL divergence (CKL) and Copula Entropy Gap (CED).

\begin{table}[ht]
\centering
\caption{Experiment 2: Mean and 95\% confidence intervals for Copula Discrepancy (CD) and Effective Sample Size (ESS) over 100 replications.}
\label{tab:exp2_full_results}
\begin{tabular}{@{}r ccc ccc@{}}
\toprule
 & \multicolumn{3}{c}{\textbf{CD}} & \multicolumn{3}{c}{\textbf{ESS}} \\
\cmidrule(lr){2-4} \cmidrule(lr){5-7}
\textbf{Step-size $\epsilon$} & \textbf{Mean} & \textbf{Lower} & \textbf{Upper} & \textbf{Mean} & \textbf{Lower} & \textbf{Upper} \\
 & & \textbf{95\% CI} & \textbf{95\% CI} & & \textbf{95\% CI} & \textbf{95\% CI} \\
\midrule
1.00e-05 & 5.0339e-01 & 4.8056e-01 & 5.2623e-01 & 2.9970e+00 & 2.9920e+00 & 3.0020e+00 \\
2.78e-05 & 5.0372e-01 & 4.8136e-01 & 5.2609e-01 & 2.9842e+00 & 2.9766e+00 & 2.9917e+00 \\
7.74e-05 & 5.1191e-01 & 4.8897e-01 & 5.3486e-01 & 2.9706e+00 & 2.9601e+00 & 2.9811e+00 \\
2.15e-04 & 5.0059e-01 & 4.7956e-01 & 5.2162e-01 & 2.9935e+00 & 2.9768e+00 & 3.0101e+00 \\
5.99e-04 & 5.1412e-01 & 4.9559e-01 & 5.3264e-01 & 2.9877e+00 & 2.9661e+00 & 3.0093e+00 \\
1.67e-03 & 5.0767e-01 & 4.9305e-01 & 5.2228e-01 & 3.0172e+00 & 2.9993e+00 & 3.0352e+00 \\
4.64e-03 & 4.9481e-01 & 4.8510e-01 & 5.0452e-01 & 3.0699e+00 & 3.0529e+00 & 3.0869e+00 \\
1.29e-02 & 5.0064e-01 & 4.9512e-01 & 5.0616e-01 & 3.1131e+00 & 3.1034e+00 & 3.1227e+00 \\
3.59e-02 & 4.9991e-01 & 4.9618e-01 & 5.0364e-01 & 3.1314e+00 & 3.1245e+00 & 3.1384e+00 \\
1.00e-01 & 4.9982e-01 & 4.9778e-01 & 5.0186e-01 & 3.1709e+00 & 3.1674e+00 & 3.1745e+00 \\
\bottomrule
\end{tabular}
\end{table}

\begin{table}[ht]
\centering
\caption{Experiment 2: Mean and 95\% confidence intervals for Copula KL divergence (CKL) and Copula Entropy Gap (CED) over 100 replications.}
\label{tab:exp2_ckl_ced}
\begin{tabular}{@{}r ccc ccc@{}}
\toprule
 & \multicolumn{3}{c}{\textbf{CKL}} & \multicolumn{3}{c}{\textbf{CED}} \\
\cmidrule(lr){2-4} \cmidrule(lr){5-7}
\textbf{Step-size $\epsilon$} & \textbf{Mean} & \textbf{Lower} & \textbf{Upper} & \textbf{Mean} & \textbf{Lower} & \textbf{Upper} \\
 & & \textbf{95\% CI} & \textbf{95\% CI} & & \textbf{95\% CI} & \textbf{95\% CI} \\
\midrule
1.00e-05 & 4.1467e-01 & 3.7428e-01 & 4.5506e-01 & 4.1484e-01 & 3.7446e-01 & 4.5521e-01 \\
2.78e-05 & 4.1147e-01 & 3.7555e-01 & 4.4739e-01 & 4.1183e-01 & 3.7609e-01 & 4.4757e-01 \\
7.74e-05 & 4.3066e-01 & 3.8769e-01 & 4.7363e-01 & 4.3135e-01 & 3.8831e-01 & 4.7439e-01 \\
2.15e-04 & 4.0308e-01 & 3.6900e-01 & 4.3715e-01 & 4.0355e-01 & 3.6950e-01 & 4.3760e-01 \\
5.99e-04 & 4.2047e-01 & 3.8820e-01 & 4.5274e-01 & 4.2098e-01 & 3.8861e-01 & 4.5336e-01 \\
1.67e-03 & 4.0015e-01 & 3.7629e-01 & 4.2401e-01 & 4.0081e-01 & 3.7691e-01 & 4.2470e-01 \\
4.64e-03 & 3.7252e-01 & 3.5726e-01 & 3.8779e-01 & 3.7305e-01 & 3.5772e-01 & 3.8837e-01 \\
1.29e-02 & 3.7772e-01 & 3.6887e-01 & 3.8656e-01 & 3.7863e-01 & 3.6986e-01 & 3.8740e-01 \\
3.59e-02 & 3.7576e-01 & 3.6966e-01 & 3.8186e-01 & 3.7631e-01 & 3.7033e-01 & 3.8229e-01 \\
1.00e-01 & 3.7564e-01 & 3.7231e-01 & 3.7896e-01 & 3.7540e-01 & 3.7212e-01 & 3.7868e-01 \\
\bottomrule
\end{tabular}
\end{table}

Tables~\ref{tab:exp3_full_results_part1} and~\ref{tab:exp3_full_results_part2} provide the detailed numerical results for the tail-dependence experiment in Figures~\ref{fig:exp3} and~\ref{fig:exp3_shannon}. For each sample size and each diagnostic (Naive Tau Discrepancy, MLE-based Copula Discrepancy (CD), Kernel Stein Discrepancy (KSD), Copula KL (CKL), and Copula Entropy Gap (CED)), we report the mean over 100 replications together with the 95\% confidence interval for the mean.

\begin{table}[h!]
\centering
\caption{Detailed numerical results for Experiment~3 (tail-dependence mismatch). For each sample size and diagnostic we report the mean and 95\% confidence interval for the mean over 100 replications. (Part 1 of 2)}
\label{tab:exp3_full_results_part1}
\resizebox{\textwidth}{!}{%
\begin{tabular}{r l c c c}
\toprule
\textbf{Sample Size} & \textbf{Metric} & \textbf{Mean} & \textbf{Lower 95\% CI} & \textbf{Upper 95\% CI} \\
\midrule
100   & Copula Discrepancy (CD)              & 1.5007e-01 & 1.3838e-01 & 1.6177e-01 \\
100   & Naive Tau Discrepancy                & 4.0651e-02 & 3.4690e-02 & 4.6611e-02 \\
100   & Kernel Stein Discrepancy (KSD)       & 2.1127e+01 & 1.7039e+01 & 2.5214e+01 \\
100   & Copula KL (CKL, Fit$||$Target)       & 1.3693e-01 & 1.1845e-01 & 1.5540e-01 \\
100   & Copula Entropy Gap (CED)             & 2.7829e-01 & 2.5997e-01 & 2.9660e-01 \\
\midrule
138   & Copula Discrepancy (CD)              & 1.4953e-01 & 1.4140e-01 & 1.5767e-01 \\
138   & Naive Tau Discrepancy                & 2.9414e-02 & 2.4608e-02 & 3.4220e-02 \\
138   & Kernel Stein Discrepancy (KSD)       & 2.6369e+01 & 2.2774e+01 & 2.9964e+01 \\
138   & Copula KL (CKL, Fit$||$Target)       & 1.2893e-01 & 1.1637e-01 & 1.4149e-01 \\
138   & Copula Entropy Gap (CED)             & 2.8096e-01 & 2.6816e-01 & 2.9375e-01 \\
\midrule
193   & Copula Discrepancy (CD)              & 1.6128e-01 & 1.5445e-01 & 1.6811e-01 \\
193   & Naive Tau Discrepancy                & 2.5994e-02 & 2.1952e-02 & 3.0036e-02 \\
193   & Kernel Stein Discrepancy (KSD)       & 4.3988e+01 & 3.9542e+01 & 4.8435e+01 \\
193   & Copula KL (CKL, Fit$||$Target)       & 1.4416e-01 & 1.3251e-01 & 1.5581e-01 \\
193   & Copula Entropy Gap (CED)             & 3.0051e-01 & 2.9023e-01 & 3.1079e-01 \\
\midrule
268   & Copula Discrepancy (CD)              & 1.6037e-01 & 1.5432e-01 & 1.6641e-01 \\
268   & Naive Tau Discrepancy                & 2.3573e-02 & 2.0327e-02 & 2.6819e-02 \\
268   & Kernel Stein Discrepancy (KSD)       & 5.6670e+01 & 5.1639e+01 & 6.1701e+01 \\
268   & Copula KL (CKL, Fit$||$Target)       & 1.4158e-01 & 1.3185e-01 & 1.5131e-01 \\
268   & Copula Entropy Gap (CED)             & 2.9914e-01 & 2.8974e-01 & 3.0853e-01 \\
\midrule
372   & Copula Discrepancy (CD)              & 1.6546e-01 & 1.6080e-01 & 1.7012e-01 \\
372   & Naive Tau Discrepancy                & 1.7451e-02 & 1.4948e-02 & 1.9954e-02 \\
372   & Kernel Stein Discrepancy (KSD)       & 7.6095e+01 & 7.0728e+01 & 8.1462e+01 \\
372   & Copula KL (CKL, Fit$||$Target)       & 1.4804e-01 & 1.4034e-01 & 1.5575e-01 \\
372   & Copula Entropy Gap (CED)             & 3.0816e-01 & 3.0119e-01 & 3.1513e-01 \\
\midrule
517   & Copula Discrepancy (CD)              & 1.6184e-01 & 1.5788e-01 & 1.6579e-01 \\
517   & Naive Tau Discrepancy                & 1.5880e-02 & 1.3572e-02 & 1.8189e-02 \\
517   & Kernel Stein Discrepancy (KSD)       & 9.3005e+01 & 8.6670e+01 & 9.9340e+01 \\
517   & Copula KL (CKL, Fit$||$Target)       & 1.4164e-01 & 1.3513e-01 & 1.4815e-01 \\
517   & Copula Entropy Gap (CED)             & 3.0262e-01 & 2.9670e-01 & 3.0854e-01 \\
\midrule
719   & Copula Discrepancy (CD)              & 1.6702e-01 & 1.6302e-01 & 1.7103e-01 \\
719   & Naive Tau Discrepancy                & 1.2829e-02 & 1.0615e-02 & 1.5043e-02 \\
719   & Kernel Stein Discrepancy (KSD)       & 1.2447e+02 & 1.1705e+02 & 1.3189e+02 \\
719   & Copula KL (CKL, Fit$||$Target)       & 1.5032e-01 & 1.4339e-01 & 1.5726e-01 \\
719   & Copula Entropy Gap (CED)             & 3.1023e-01 & 3.0423e-01 & 3.1622e-01 \\
\midrule
1000  & Copula Discrepancy (CD)              & 1.6352e-01 & 1.6018e-01 & 1.6685e-01 \\
1000  & Naive Tau Discrepancy                & 1.2292e-02 & 1.0507e-02 & 1.4078e-02 \\
1000  & Kernel Stein Discrepancy (KSD)       & 1.4464e+02 & 1.3756e+02 & 1.5171e+02 \\
1000  & Copula KL (CKL, Fit$||$Target)       & 1.4398e-01 & 1.3859e-01 & 1.4937e-01 \\
1000  & Copula Entropy Gap (CED)             & 3.0507e-01 & 3.0001e-01 & 3.1012e-01 \\
\bottomrule
\end{tabular}
}
\end{table}

\begin{table}[h!]
\centering
\caption{Detailed numerical results for Experiment~3 (tail-dependence mismatch). For each sample size and diagnostic we report the mean and 95\% confidence interval for the mean over 100 replications. (Part 2 of 2)}
\label{tab:exp3_full_results_part2}
\resizebox{\textwidth}{!}{%
\begin{tabular}{r l c c c}
\toprule
\textbf{Sample Size} & \textbf{Metric} & \textbf{Mean} & \textbf{Lower 95\% CI} & \textbf{Upper 95\% CI} \\
\midrule
1389  & Copula Discrepancy (CD)              & 1.6841e-01 & 1.6550e-01 & 1.7133e-01 \\
1389  & Naive Tau Discrepancy                & 1.1017e-02 & 9.4294e-03 & 1.2605e-02 \\
1389  & Kernel Stein Discrepancy (KSD)       & 1.7406e+02 & 1.6486e+02 & 1.8327e+02 \\
1389  & Copula KL (CKL, Fit$||$Target)       & 1.5151e-01 & 1.4648e-01 & 1.5655e-01 \\
1389  & Copula Entropy Gap (CED)             & 3.1319e-01 & 3.0887e-01 & 3.1752e-01 \\
\midrule
1930  & Copula Discrepancy (CD)              & 1.6665e-01 & 1.6431e-01 & 1.6899e-01 \\
1930  & Naive Tau Discrepancy                & 8.4308e-03 & 6.9829e-03 & 9.8787e-03 \\
1930  & Kernel Stein Discrepancy (KSD)       & 2.0563e+02 & 1.9683e+02 & 2.1443e+02 \\
1930  & Copula KL (CKL, Fit$||$Target)       & 1.4877e-01 & 1.4479e-01 & 1.5276e-01 \\
1930  & Copula Entropy Gap (CED)             & 3.1044e-01 & 3.0692e-01 & 3.1397e-01 \\
\midrule
2682  & Copula Discrepancy (CD)              & 1.6780e-01 & 1.6559e-01 & 1.7002e-01 \\
2682  & Naive Tau Discrepancy                & 7.4509e-03 & 6.3376e-03 & 8.5642e-03 \\
2682  & Kernel Stein Discrepancy (KSD)       & 2.4402e+02 & 2.3451e+02 & 2.5353e+02 \\
2682  & Copula KL (CKL, Fit$||$Target)       & 1.5025e-01 & 1.4640e-01 & 1.5411e-01 \\
2682  & Copula Entropy Gap (CED)             & 3.1216e-01 & 3.0884e-01 & 3.1548e-01 \\
\midrule
3727  & Copula Discrepancy (CD)              & 1.6890e-01 & 1.6707e-01 & 1.7072e-01 \\
3727  & Naive Tau Discrepancy                & 5.9265e-03 & 4.9689e-03 & 6.8842e-03 \\
3727  & Kernel Stein Discrepancy (KSD)       & 2.7959e+02 & 2.6978e+02 & 2.8941e+02 \\
3727  & Copula KL (CKL, Fit$||$Target)       & 1.5229e-01 & 1.4908e-01 & 1.5550e-01 \\
3727  & Copula Entropy Gap (CED)             & 3.1363e-01 & 3.1091e-01 & 3.1636e-01 \\
\midrule
5179  & Copula Discrepancy (CD)              & 1.6909e-01 & 1.6751e-01 & 1.7067e-01 \\
5179  & Naive Tau Discrepancy                & 5.1188e-03 & 4.3433e-03 & 5.8944e-03 \\
5179  & Kernel Stein Discrepancy (KSD)       & 3.4070e+02 & 3.3196e+02 & 3.4944e+02 \\
5179  & Copula KL (CKL, Fit$||$Target)       & 1.5173e-01 & 1.4888e-01 & 1.5458e-01 \\
5179  & Copula Entropy Gap (CED)             & 3.1377e-01 & 3.1133e-01 & 3.1622e-01 \\
\midrule
7196  & Copula Discrepancy (CD)              & 1.6937e-01 & 1.6810e-01 & 1.7063e-01 \\
7196  & Naive Tau Discrepancy                & 4.7903e-03 & 4.1026e-03 & 5.4780e-03 \\
7196  & Kernel Stein Discrepancy (KSD)       & 3.8051e+02 & 3.7137e+02 & 3.8966e+02 \\
7196  & Copula KL (CKL, Fit$||$Target)       & 1.5259e-01 & 1.5027e-01 & 1.5492e-01 \\
7196  & Copula Entropy Gap (CED)             & 3.1481e-01 & 3.1287e-01 & 3.1675e-01 \\
\midrule
10000 & Copula Discrepancy (CD)              & 1.6935e-01 & 1.6832e-01 & 1.7039e-01 \\
10000 & Naive Tau Discrepancy                & 3.4350e-03 & 2.9233e-03 & 3.9467e-03 \\
10000 & Kernel Stein Discrepancy (KSD)       & 4.2903e+02 & 4.1828e+02 & 4.3978e+02 \\
10000 & Copula KL (CKL, Fit$||$Target)       & 1.5256e-01 & 1.5063e-01 & 1.5449e-01 \\
10000 & Copula Entropy Gap (CED)             & 3.1464e-01 & 3.1305e-01 & 3.1623e-01 \\
\bottomrule
\end{tabular}
}
\end{table}

\end{document}